\documentclass[11pt,a4paper]{amsart}
\usepackage{fullpage}

\usepackage{amsthm,amsmath,amssymb}
\usepackage{booktabs}
\usepackage{multirow}
\usepackage{makecell}
\usepackage{tikz}
\usepackage{mathtools}
\usetikzlibrary{arrows.meta}
\usetikzlibrary{calc}

\usepackage[bookmarks=true,hypertexnames=false,pagebackref]{hyperref}
\hypersetup{colorlinks=true, citecolor=blue, linkcolor=red, urlcolor=blue}
\usepackage{amsthm}
\usepackage[lining,semibold,type1]{libertine} % a bit lighter than Times--no osf in math
\usepackage[T1]{fontenc} % best for Western European languages
\usepackage{textcomp} % required to get special symbols
\usepackage[varqu,varl]{inconsolata}% a typewriter font must be defined
\usepackage[libertine,vvarbb]{newtxmath}
\usepackage[scr=rsfso]{mathalfa}
\usepackage{bm}
\usepackage{cleveref}
\usepackage{graphicx}
\usepackage{cite}
\usepackage{enumerate}
\usepackage{bbm}
\usepackage{xifthen}
\usepackage{todonotes}

\newtheorem{theorem}{Theorem}
\newtheorem{lemma}[theorem]{Lemma}
\newtheorem{corollary}[theorem]{Corollary}
\newtheorem{definition}[theorem]{Definition}

\newtheorem*{remark}{Remark}

\newtheorem*{discussion}{Discussion}

\newcommand{\norm}[1]{\left\Vert#1\right\Vert}
\newcommand{\abs}[1]{\left\vert#1\right\vert}
\newcommand{\set}[1]{\left\{#1\right\}}
\newcommand{\tuple}[1]{\left(#1\right)} \newcommand{\eps}{\varepsilon}
\newcommand{\inner}[2]{\langle #1,#2\rangle} \newcommand{\tp}{\tuple}
\renewcommand{\mid}{\;\middle\vert\;} \newcommand{\cmid}{\,:\,}
 \renewcommand{\P}{\mathbf{P}}
\newcommand{\defeq}{\triangleq} 
\newcommand{\ol}{\overline}

\DeclareMathOperator{\dom}{dom}
\DeclareMathOperator{\interior}{int}

\DeclareMathOperator*{\argmin}{arg\,min}

\renewcommand{\P}[2][]{ \ifthenelse{\isempty{#1}}
	{\mathbf{P}_{#2}}{\mathbf{P}_{#1}(#2)} }

\newcommand{\ind}[1]{\mathbf{1}\left[#1\right]}

\usepackage[ruled,vlined]{algorithm2e}

\def\*#1{\mathbf{#1}} \def\+#1{\mathcal{#1}} \def\-#1{\mathrm{#1}} \def\^#1{\mathbb{#1}}

\renewcommand{\Pr}[2][]{ \ifthenelse{\isempty{#1}}
  {\mathbf{Pr}\left[#2\right]} {\mathbf{Pr}_{#1}\left[#2\right]} }
\newcommand{\E}[2][]{ \ifthenelse{\isempty{#1}}
  {\mathbf{E}\left[#2\right]}
  {\operatorname{\mathbf{E}}_{#1}\left[#2\right]} }
\newcommand{\Var}[2][]{ \ifthenelse{\isempty{#1}}
  {\mathbf{Var}\left[#2\right]}
  {\mathop{\mathbf{Var}}_{#1}\left[#2\right]} }

\usepackage{varwidth}
\newcommand{\nprob}[4][8]{
  \begin{center}\normalfont\fbox{
      \begin{varwidth}{\textwidth}
        \textsc{#2}\\
        \hspace*{1.5em}
        \begin{tabular}[t]{rp{#1cm}}%\textsc{#2} \
          \textit{Input:}&#3\\
          \textit{Problem:}&#4
        \end{tabular}
      \end{varwidth}
    }
  \end{center}
}

\usepackage{xspace}
\newcommand{\Ber}[1]{\mathrm{Ber}\tp{#1}}
\newcommand{\BAI}{\textsc{BestArmId}\xspace}

\title{Understanding Bandits with Graph Feedback}

\author{Houshuang Chen}
\address[Houshuang Chen]{Shanghai Jiao Tong University, No.800 Dongchuan Road, Minhang District, Shanghai, China.}
\email{chenhoushuang@sjtu.edu.cn}

\author{Zengfeng Huang}
\address[Zengfeng Huang]{Fudan University, No.220 Handan Road, Yangpu District, Shanghai, China.}
\email{huangzf@fudan.edu.cn}
\author{Shuai Li}
\address[Shuai Li]{Shanghai Jiao Tong University, No.800 Dongchuan Road, Minhang District, Shanghai, China.}
\email{shuaili8@sjtu.edu.cn}

\author{Chihao Zhang}
\address[Chihao Zhang]{Shanghai Jiao Tong University, No.800 Dongchuan Road, Minhang District, Shanghai, China.}
\email{chihao@sjtu.edu.cn}

\begin{document}
\begin{abstract}
	The bandit problem with graph feedback, proposed in~[Mannor and Shamir, NeurIPS 2011], is modeled by a directed graph $G=(V,E)$ where $V$ is the collection of bandit arms, and once an arm is triggered, all its incident arms are observed. A fundamental question is how the structure of the graph affects the min-max regret. We propose the notions of the fractional weak domination number $\delta^*$ and the $k$-packing independence number capturing upper bound and lower bound for the regret respectively.  We show that the two notions are inherently connected via aligning them with the linear program of the weakly dominating set and its dual --- the fractional vertex packing set respectively. Based on this connection, we utilize the strong duality theorem to prove a general regret upper bound $O\tp{\tp{\delta^*\log\abs{V}}^{\frac{1}{3}}T^{\frac{2}{3}}}$ and a lower bound $\Omega\tp{\tp{\delta^*/\alpha}^{\frac{1}{3}}T^{\frac{2}{3}}}$ where $\alpha$ is the integrality gap of the dual linear program. Therefore, our bounds are tight up to a $\tp{\log\abs{V}}^{\frac{1}{3}}$ factor on graphs with bounded integrality gap for the vertex packing problem including trees and graphs with bounded degree. Moreover, we show that for several special families of graphs, we can get rid of the $\tp{\log\abs{V}}^{\frac{1}{3}}$ factor and establish optimal regret.
\end{abstract}
\maketitle
\section{Introduction}

The multi-armed bandit is an extensively studied problem in reinforcement learning. Imagining a player facing an $n$-armed bandit, each time the player pulls one of the $n$ arms and incurs a loss. At the end of each round, the player receives some feedback and tries to make a better choice in the next round. The expected regret is defined by the difference between the player's cumulative losses and cumulative losses of the single best arm during $T$ rounds. In this article, we assume the loss at each round is given in an adversarial fashion. This is called the \emph{adversarial bandit} in the literature. The difficulty of the adversarial bandit problem is usually measured by the min-max regret which is the expected regret of the best strategy against the worst possible loss sequence.

	Player's strategy depends on how the feedback is given at each round. One simple type of feedback is called \emph{full feedback} where the player can observe all arm's losses after playing an arm. An important problem studied in this model is \emph{online learning with experts}~\cite{cesa2006prediction,eban2012learning}. Another extreme is the vanilla \emph{bandit feedback} where the player can only observe the loss of the arm he/she just pulled~\cite{auer2002finite}. Optimal bounds for the regret, either in $n$ or in $T$, are known for both types of feedback.
	
	The work of~\cite{mannor2011bandits} initialized the study on the generalization of the above two extremes, that is, the feedback at each round consists of the losses of a collection of arms. This type of feedback can be naturally described by a \emph{feedback graph} $G$ where the vertex set is $[n]$ and a directed edge $(i,j)$ means pulling the arm $i$ can observe the loss of arm $j$. Therefore, the ``full feedback'' means that $G$ is a clique with self-loops, and the ``vanilla bandit feedback'' means that $G$ consists of $n$ disjoint self-loops. 
	
	A natural yet challenging question is how the graph structure affects the min-max regret. The work of~\cite{ACDK15} systematically investigated the question and proved tight regret bounds in terms of the time horizon $T$. They show that, if the graph is ``strongly observable'', the regret is $\Theta(T^{\frac{1}{2}})$; if the graph is ``weakly observable'', the regret is $\Theta(T^{\frac{2}{3}})$; and if the graph is ``non-observable'', the regret is $\Theta(T)$. Here the notions of ``strongly observable'', ``weakly observable'' and ``non-observable'' roughly indicate the connectivity of the feedback graph and will be formally defined in our \Cref{sec:pre}. However, unlike the case of ``full feedback'' or ``vanilla bandit feedback'', the dependency of the regret on $n$, or more generally on the structure of the graph, is still not well understood. For example, for ``weakly observable'' graphs, an upper bound and a lower bound of the regret in terms of the weak domination number $\delta(G)$ were proved in~\cite{ACDK15}, but a large gap exists between the two. This suggests that the weak domination number might not be the correct parameter to characterize the regret.
	
	We make progress on this problem for ``weakly observable'' graphs. This family of graphs is general enough to encode almost all feedback patterns of bandits. Consequently, tools in graph theory are useful to understand its rich structures. We introduce the notions of the fractional weak domination number $\delta^*(G)$ and the $k$-packing independence number, and then provide evidence to show that they are the correct graph parameters. The two parameters are closely related and help us to improve the upper bound and lower bound respectively. As the name indicated, $\delta^*(G)$ is the fractional version of $\delta(G)$, namely the optimum of the linear relaxation of the integer program for the weakly dominating set. We observe that this graph parameter has already been used in an algorithm for ``strongly observable'' graphs in~\cite{alon2017nonstochastic}, where it functioned differently. In the following, when the algorithm is clear from the context, we use $R(G,T)$ to denote the regret of the algorithm on the instance $G$ in $T$ rounds. Our main algorithmic result is:
	
\begin{theorem}\label{thm:regret}
	There exists an algorithm such that for any weakly observable graph, any time horizon $T\geq n^3\log(n)/{\delta^*}^2(G)$, its regret satisfies
	\[
	 R(G,T)=O\tp{\tp{\delta^*(G)\log n}^{\frac{1}{3}}T^{\frac{2}{3}}}\,.
	\]
\end{theorem}
	
Note that the regret of the algorithm in~\cite{ACDK15} satisfies $R(G,T)=O\tp{\tp{\delta(G)\log n}^{\frac{1}{3}}T^{\frac{2}{3}}}$. The fractional weak domination number $\delta^*$ is always no larger than $\delta$, and the gaps between the two can be as large as $\Theta(\log n)$. We will give an explicit example in \Cref{sec:instance-lb} in which the gap matters and our algorithm is optimal. \Cref{thm:regret} can be seamlessly extended  to more general time-varying graphs and probabilistic graphs.  The formal definitions of these models are in \Cref{sec:generalization}.

On the other hand, we investigate graph structures that can be used to fool algorithms. We say a set $S$  of vertices is a $k$-packing independent set if $S$ is an independent set and any vertex has at most $k$ out-neighbors in $S$. We prove the following lower bound:

\begin{theorem}\label{thm:meta-lb}
	Let $G=(V,E)$ be a directed graph. If $G$ contains a $k$-packing independent set $S$ with $\abs{S}\ge 2$, then for any randomized algorithm and any time horizon $T$, there exists a sequence of loss functions such that the expected regret is $\Omega\tp{\max\set{\frac{|S|}{k},\log |S|}^{\frac{1}{3}}\cdot T^{\frac{2}{3}}}$. 
\end{theorem}

For every $k\in\^N$, we use $\zeta_k$, the $k$-packing independence number, to denote the size of the maximum $k$-packing independent set. To prove \Cref{thm:meta-lb}, we reduce the problem of minimizing regret to \emph{statistical hypothesis testing} for which powerful tools from information theory can help. 

We can use \Cref{thm:meta-lb} to strengthen lower bounds in~\cite{ACDK15}. Besides, we show that large $\delta^*$ usually implies large $\zeta_1$ via studying the linear programming dual of fractional weakly dominating sets and applying a novel rounding procedure. This is also one of our main technical contributions.. Combinatorially, the dual linear program is the problem of finding the maximum fractional vertex packing set in the graph. Therefore, we can establish lower bounds in terms of $\delta^*$ by applying \Cref{thm:meta-lb}:

\begin{theorem} \label{thm:dual-lb}
If $G$ is weakly observable, then for any algorithm and any sufficiently large time horizon $T\in\^N$, there exists a sequence of loss functions such that
	\[
		R(G,T)=\Omega\tp{\tp{\frac{\delta^*}{\alpha}}^{1/3}\cdot T^\frac{2}{3}}\,,
	\]
	where $\alpha$ is the integrality gap of the linear program for vertex packing.
\end{theorem}

Clearly the exact lower bound is determined by the integrality gap of a certain linear program. In general graphs, we have a universal upper bound $\alpha=O\tp{n/\delta^*}$. For concrete instances, we can obtain clearer and tighter bounds on $\alpha$. For example, the linear program has a constant integrality gap $\alpha$ on graphs of bounded degree.

\begin{corollary}\label{cor:bounded-degree}
	Let $\Delta\in\^N$ be a constant and $\+G_\Delta$ be the family of graphs with maximum in-degree $\Delta$. Then for every weakly observable $G=(V,E)\in\+G_\Delta$, any algorithm and any sufficiently large time horizon $T\in\^N$,  there exists a sequence of loss functions such that
	\[
	R(G,T)=\Omega((\delta^*)^{\frac{1}{3}}\cdot T^{\frac{2}{3}})\,.
	\]
\end{corollary}

We also show that for those $1$-degenerate directed graphs (formally defined in \Cref{sec:lp}), the integrality gap is $1$. This family of graphs includes trees and directed cycles. As a consequence, we have

\begin{corollary}\label{cor:degen}
	Let $G$ be a $1$-degenerate weakly observable graph. Then for any algorithm and any sufficiently large time horizon $T\in\^N$,  there exists a sequence of loss functions such that
	\[
	R(G,T)=\Omega((\delta^*)^{\frac{1}{3}}\cdot T^{\frac{2}{3}}).
	\]
\end{corollary}

%Therefore, the upper bound in \Cref{thm:regret} and the lower bound \Cref{thm:dual-lb} are tight up to a $\tp{\log n}^{\frac{1}{3}}$ factor on trees and on graphs with bounded degree. Interestingly, for some instances, we can even get rid of the  $\tp{\log n}^{\frac{1}{3}}$ factor by directly applying \Cref{thm:meta-lb} and thus obtain optimal regret. We will discuss these graphs in \Cref{sec:instance-lb}.
%
%A comparison of previous results and our results is summarized in \Cref{tab:compare}. 

\subsection*{Comparison of previous results and our results}

In \Cref{tab:compare}, we compare our new upper bounds, lower bounds and their gap with previous best results.

\begin{table}[!htbp]
	\caption{A comparison of results}
	\label{tab:compare}
%	\centering
	\resizebox{\width}{\height}{
		\begin{tabular}{ccccc}
			\toprule
			\multirow{2}{*}{Graph Type} &\multicolumn{2}{c}{Previous best results \cite{ACDK15}} &\multicolumn{2}{c}{This work}\\
			\cmidrule(r){2-3}\morecmidrules\cmidrule{4-5} &Min-max regret & Gap & Min-max regret & Gap \\
			\midrule
			\multirow{2}{*}{General graphs} & \makecell[l]{$O\tp{(\delta\log n)^{\frac{1}{3}}\cdot T^{\frac{2}{3}}}$\\  $\Omega\tp{\max\big\{(\frac{\delta}{\log^2n})^{\frac{1}{3}},1\big\}\cdot T^{\frac{2}{3}}}$}& \makecell{See discussion\\ below}  & \makecell[l]{$O\tp{(\delta^*\log n)^{\frac{1}{3}}\cdot T^{\frac{2}{3}}}$\\ $\Omega\tp{\max\big\{(\frac{\delta^*}{\alpha})^{\frac{1}{3}},1\big\}\cdot T^{\frac{2}{3}}}$}  & \makecell{See discussion\\below }\\
			\cmidrule{2-5}
			& 	{\small for $\delta=\log^2n$}: $\Omega\tp{T^{\frac{2}{3}}}$&$O\tp{\log n}$ &\makecell[l]{{\small for $\delta=\log^2n$}:\\ {\quad  } $\Omega\tp{\log\log n\cdot T^{\frac{2}{3}}}$}& $O\tp{\frac{\log n}{\log\log n}}$
			\\
			\midrule
			\makecell{ Trees /\\ Bounded in-degree} & Same as general graphs & $O\tp{\log n}$& \makecell[l]{$O\tp{(\delta^*\log n)^{1/3}\cdot T^{2/3}}$\\ $\Omega\tp{(\delta^*)^{1/3}\cdot T^{2/3}}$} & $O\tp{\tp{\log n}^{1/3}}$\\
			\midrule 
			\makecell{Complete \\bipartite graphs}& \makecell[l]{$O\tp{(\log n)^{1/3}\cdot T^{2/3}}$\\ $\Omega\tp{T^{2/3}}$} &$O\tp{\tp{\log n}^{\frac{1}{3}}}$ & $\Theta\tp{(\log n)^{1/3}\cdot T^{2/3}}$& $O\tp{1}$\\
			\midrule
			\makecell{Orthogonal \\relation on $\mathbb{F}_2^k$} &   \makecell[l]{$O\tp{(\log n)^{2/3}\cdot T^{2/3}}$\\ $\Omega\tp{T^{2/3}}$}&$O\tp{\tp{\log n}^{\frac{2}{3}}}$& $\Theta\tp{(\log n)^{1/3}\cdot T^{2/3}}$&$O\tp{1}$\\
			\bottomrule
	\end{tabular}}
	%	\bigskip
\end{table} 
%\end{center}
 \begin{discussion}\label{dis}
 	In general, our upper bound is never worse than the previous one since $\delta^*\le \delta$.
 	Our lower bound is not directly comparable to the previously known lower bound as they are stated
 	in terms of different parameters. In fact, we can not find an instance such that our lower bound
 	$\Omega(\max\{1,(\delta^*/\alpha)^{1/3}\})$ is worse than the previous lower bound $\Omega(\max\{1,(\delta/(\log n)^2)^{1/3}\})$ and
 	there are instances on which our bound outperforms. The two key quantities, namely the integrality
 	gap $\frac{\delta}{\delta^*}$ of the primal linear programming and the integrality gap $\alpha$ of the dual linear programming, seem to be correlated in a graph. The relation between the two bounds is worth further investigation.
 \end{discussion}
 
\subsection*{Related Work}
The multi-armed bandit problem originated from the sequential decision making under uncertainty studied in~\cite{wald1947sequential,arrow1949bayes} and the adversarial bandit is a natural variant introduced by~\cite{auer1998line}. The work of~\cite{mannor2011bandits} introduced the graph feedback model with a self-loop on each vertex in order to interpolating between the full feedback and bandit feedback settings. This model has been extensively studied in the work of~\cite{mannor2011bandits,alon2013bandits,kocak2014efficient,alon2017nonstochastic}. The work of~\cite{ACDK15} removed the self-loop assumption and considered generalized constrained graphs. They gave a full characterization of the mini-max regret in terms of the time horizon $T$. In contrast to fixed graph feedback, recent work of~\cite{kocak2014efficient,cohen2016online,alon2015online,tossou2017thompson} considered the time-varying graphs. Another line of recent work in~\cite{li2020stochastic,liu2018information,alon2017nonstochastic} is to study random graphs, or the graphs with probabilistic feedback.

Most algorithms for adversarial bandits are derived from the EXP3 algorithm, e.g.~\cite{auer2002nonstochastic,radlinski2008learning}. However, even for the vanilla multi-armed bandit problem, a direct application of EXP3 can only get an upper bound of $O\tp{\sqrt{n\log n \cdot T}}$~\cite{auer2002nonstochastic}, which is suboptimal. In fact, the EXP3 is a special case of \emph{online stochastic mirror descent} (OSMD) algorithm when using negentropy function as the potential function. OSMD was developed by ~\cite{nemirovski1979efficient,nemirovskij1983problem} for online optimization. By choosing a more sophisticated potential function, namely the $1/2$-Tsallis entropy function~\cite{tsallis1988possible}, OSMD can achieve the tight bound $\Theta\tp{\sqrt{nT}}$~\cite{ZL19}.

	The idea of using domination number or related parameters to study the feedback graph appeared many times in literature, e.g.~\cite{ACDK15,buccapatnam2013multi,buccapatnam2014stochastic,alon2013bandits,alon2017nonstochastic}. The work of ~\cite{alon2017nonstochastic} used the idea of the fractional dominating set to study the high-probability regret bound for the strongly observable graph. Other similar works \cite{liu2018information,tossou2017thompson,buccapatnam2014stochastic} mainly focused on stochastic settings. The follow-up works related to the weakly observable graph mainly considered harder settings including the time-varying graphs~\cite{alon2015online,cohen2016online,alon2017nonstochastic}, bounded-memory adversaries~\cite{feng2018online} and the feedback graphs with switching costs~\cite{rangi2019online,arora2019bandits}. The recent work of~\cite{lee2020closer} considered the bound with respect to cumulative losses of the best arm. To the best of our knowledge, there is no further development on the min-max regret since the work of~\cite{ACDK15}.

The paper is organized as follows. In \Cref{sec:pre}, we set up our framework and introduce tools that will be used afterward. We introduce our algorithm in \Cref{sec:alg} and analyze it in \Cref{sec:ub-proof}. In \Cref{sec:lb}, we prove the meta lower bound in terms of $k$-packing independent set and derive our other lower bounds from it. Then we prove the meta lower bound in \Cref{sec:meta-lb}. In \Cref{sec:generalization}, we extend the deterministic graph to the more general time-varying graphs and probabilistic graphs. Finally, we conclude the article and discuss some open problems.
%\todo{1. lower bounds for stochastic; 2. strongly observable}
\section{Preliminaries}\label{sec:pre}

In this section, we formally describe the problem setting of bandits with graph feedback and introduce notations, definitions and propositions that will be used later.

Let $n\in \^N$ be a nonnegative integer. We will use $[n]$ to denote the set $\set{1,2,\dots,n}$. Let $\*x\in\^R^n$ be an $n$-dimensional vector. For every $i\in[n]$, we use $\*x(i)$ to denote the value on the $i$th-coordinate. We use $\set{\*e_1,\dots,\*e_n}$ to denote the standard basis of $\^R^n$. That is, $\*e_i\in\^R^n$ is the vector such that 
$
\*e_i(j)=
\begin{cases}
1, & \mbox{ if }j=i\\
0, & \mbox{ if }j\ne i
\end{cases}
$ for every $j\in [n]$.

For every $n\in \^N$,  we define $\Delta_n\defeq\set{\*x\in\^R_{\ge 0}^n\cmid \sum_{i=1}^n\*x(i)=1}$ as the $n$-dimensional probability simplex. Clearly $\Delta_n$ is convex and every $\*x\in \Delta_n$ can be viewed as a distribution on $[n]$. Throughout the article, sometimes we will view a function $\ell:[n]\to \^R$ equivalently as a vector in $\^R^n$, depending on which form is more convenient in the context. With this in mind, we have the inner product $\inner{\ell}{\*x} \defeq \sum_{i\in[n]}\ell(i)\cdot \*x(i)$ for every $\*x\in\^R^n$.

\subsection{Graphs}\label{sec:lp}

In this article, we use $G=(V,E)$ to denote a directed graph with possible self-loops but no multiple edges. Therefore each $(u,v)\in E$ indicates a directed edge from $u$ to $v$ in $G$. If we say a graph $G=(V,E)$ is undirected, we view each undirected edge $\set{u,v}\in E$ as two directed edges $(u,v)$ and $(v,u)$. In the following, we assume $\abs{V}\ge 2$ unless otherwise specified. For any $S\subseteq V$, $G[S]$ is the subgraph of $G$ induced by $S$. For every $v\in V$, we define $N_{\-{in}}(v)=\set{u\in V\cmid (u,v)\in E}$ and $N_{\-{out}}(v)=\set{u\in V\cmid (v,u)\in E}$ as the set of in-neighbors and out-neighbors of $v$ respectively. We also call $\abs{N_{\-{in}}(v)}$ and $\abs{N_{\-{out}}(v)}$  the in-degree and out-degree of $v$ respectively. A set $S\subseteq V$ is an \emph{independent set} if there is no $u,v\in S$ such that $(u,v)\in E$. Note that we \emph{do not} consider an isolated vertex with a self-loop as an independent set.

	A vertex $v\in V$ is \emph{strongly observable} if $(v,v)\in E$ or $\forall u\in V\setminus{v}$, $(u,v)\in E$. A vertex $v\in V$ is \emph{non-observable} if $N_\-{in}(v)=\varnothing$. 	A directed graph $G$ is called \emph{strongly observable} if each vertex of $G$ is strongly observable. It is called \emph{non-observable} if it contains at least one non-observable vertex. The graph is called \emph{weakly observable} if it is neither strong observable nor non-observable.
	
	We say a directed graph $G$ is $1$-degenerate if one can iteratively apply the following operators in arbitrary orders on $G$ to get an empty graph:
	\begin{itemize}
		\item Pick a vertex with in-degree one and remove the in-edge;
		\item Pick a vertex with in-degree zero and out-degree at most one, and remove both the vertex and the out-edge.
	\end{itemize}

	Typical $1$-degenerate graphs include trees (directed or undirected) and directed cycles.
	
\smallskip
	Let $U=\{i\in V:i\notin N_{\text{in}}(i)\}$ denote the set of vertices without self-loops. Consider the following linear program defined on $G$. We will call the linear program $\mathscr{P}$.

\begin{center}
	\begin{minipage}{0.6\textwidth}
		\begin{alignat}{3}
			\notag &\text{minimize} & \sum_{i\in V} x_i &\\
			\tag{$\mathscr{P}$}&\text{subject to} \quad & \sum_{i\in N_{\mathrm{in}}(j)}  x_i&\ge 1, &\quad \forall j\in U\\
			\notag& & 0\le x_i&\le 1, &\quad \forall i\in V
		\end{alignat}
	\end{minipage}
\end{center}

We use $\delta^*(G)$ to denote the optimum of the above linear program. We call $\delta^*(G)$ the \emph{fractional weak domination number} of $G$. The dual of the linear program is

\smallskip

\begin{center}
	\begin{minipage}{0.6\textwidth}
		\begin{alignat}{3}
			\notag &\text{maximize} & \sum_{j\in U} y_j &\\
			\tag{$\mathscr{D}$} &\text{subject to} \quad & \sum_{j\in N_{\mathrm{out}}(i)\cap U}  y_j&\le 1, &\quad \forall i\in V\\
			\notag & & 0\le y_j&\le 1, &\quad \forall j\in U
		\end{alignat}
	\end{minipage}
\end{center}

We call this linear program $\mathscr{D}$. We use $\zeta^*(G)$ to denote the optimum of the dual. We call $\zeta^*(G)$ the \emph{fractional vertex packing number} of $G$. Then it follows from the \emph{strong duality theorem} (see e.g.~\cite{BV04}) of linear programs that $\delta^*(G)=\zeta^*(G)$.

We remark that in \cite{ACDK15}, the weakly (integral) dominating set was defined to dominate all ``weakly observable vertices''. This was indeed a flaw in the paper as in some extreme cases, the set may fail to dominate vertices in $U$ that are ``strongly observable''. Therefore we ask for the set to dominate $U$. Nevertheless the two definitions of domination number, both integral and fractional, differ by at most one and do not affect the asymptotic bounds. More explanations on this can be found in \Cref{apx:dc}. 
%\todo{weakly v.s. strongly}
\subsection{Bandits}

Let $G=(V,E)$ be a directed graph where $V=[n]$ is the collection of bandit arms. Let $T\in\^N$ be the time horizon which is known beforehand.
The bandit problem is an online-decision game running for $T$ rounds. The player designs an algorithm $\mathscr{A}$ with the following behavior in each round $t$ of the game:
	\begin{itemize}
		\item The algorithm $\mathscr{A}$ chooses an arm $A_t\in [n]$;
		\item An adversary privately provides a loss function $\ell_t:\^N\to[0,1]$ and $\mathscr{A}$ pays a loss $\ell_t(A_t)$;
		\item The algorithm receives feedback $\set{\ell_t(j)\cmid j\in N_{\-{out}}(A_t)}$.
	\end{itemize}

The \emph{expected regret} of the algorithm $\mathscr{A}$ against a specific loss sequence $\ell^*=\{\ell_1,\dots,\ell_T\}$ is defined by $R(G, T,\mathscr{A},\ell^*)=\E{\sum_{t=1}^T\ell_{t}(A_t)}-\min_{a\in [n]}\sum_{t=1}^T\ell_{t}(a)$. Note that we look at the expectation of the algorithm since $\mathscr{A}$ might be randomized and it is not hard to see that randomization is necessary to guarantee $o(T)$ regret due to the adversarial nature of the loss sequence. 

 The purpose of the problem is to design an algorithm performing well against the worst loss sequence, namely determining the min-max regret
 \[
 	R(G, T)\defeq \inf_{\mathscr{A}}\sup_{\ell^*}R(G, T,\mathscr{A},\ell^*)\,.
 \]
 
 There is another model called \emph{stochastic bandits} in which the loss function at each round is not adversarially chosen but sampled from a fixed distribution. It is clear that this model is not harder than the one introduced above in the sense that any algorithm performing well in the adversarial setting also performs well in the stochastic setting.
 %\shuai{the definitions of regret under stochastic setting and adversarial setting are a little bit different: one uses real loss vectors and one uses expected mean. there would be a gap between these two. but this doesn't matter for a reasonable algorithm in the adversarial setting. perhaps such a statement is fine}
  Therefore, we will construct instances of stochastic bandits to derive lower bounds in \Cref{sec:lb}.

\subsection{Optimization} Our upper bound is obtained via the online mirror descent algorithm. In this section, we collect a minimal set of terminologies to understand the algorithm. More details about the approach and its application to online decision-making can be found in e.g.\ \cite{Ora2019}.

Let $C\subseteq \^R^n$. We use $\interior(C)$ to denote the interior $C$. For a convex function $\Psi:\mathbb{R}^d\to \mathbb{R}\cup \{\infty\}$, $\text{dom}(\Psi)\defeq \{x:\Psi(x)<\infty\}$ is the domain of $\Psi$. Assume $\Psi$ is differentiable in its domain. For every $\*x,\*y\in \dom\tp{\Psi}$,  $B_\Psi(\*x,\*y)\defeq \Psi(\*x)-\Psi(\*y)-\langle \*x-\*y,\nabla\Psi(\*y)\rangle\ge 0$ is the \emph{Bregman divergence} between $\*x$ and $\*y$ with respect to the convex function $\Psi$. The diameter of $C$ with respect to $\Psi$ is $D_\Psi(C) \defeq \sup_{\*x,\*y\in C}\{\Psi(\*x)-\Psi(\*y)\}$.
	
		Let $A\in \^R^{n\times n}$ be a \emph{semi-definite positive} matrix and $\*x\in \^R^n$ be a vector. We define $\|\*x\|_{A}\defeq \sqrt{\*x^{\mathtt{T}}A\*x}$ as the norm of $\*x$ with respect to $A$. 

\subsection{Information Theory}

We borrow tools from information theory to establish lower bounds. More details can be found in the standard textbook~\cite{COV99} on the topic. To ease the notation, each ``$\log$'' appeared in the article without subscript is of base $e$. We fix a probability space $(\Omega,\+F,\*{Pr})$ and let $X,Y:\Omega\to U$ be discrete-valued random variables for a finite set $U$. 

The entropy of $X$ is $H(X)\defeq -\sum_{x\in U}\Pr{X=x}\cdot \log \Pr{X=x}$. The conditional information $H(X|Y)\defeq -\sum_{x,y\in U}\Pr{X=x, Y=y}\log \Pr{X=x\mid Y=y}$. The mutual information between $X$ and $Y$ is $I(X;Y)\defeq \sum_{x,y\in U}\Pr{X=x, Y=y}\cdot\log \frac{\Pr{X=x, Y=y}}{\Pr{X=x}\cdot \Pr{Y=y}}$. It is a basic fact that $H(X)=H(X|Y)+I(X;Y)$. Suppose we have another random variable $Z$, then $I(X;Y|Z)\defeq H(X|Z)-H(X|Y,Z)$.

Suppose $Z:\Omega\to W$  is a random variable correlated to $X$ and one needs to guess the value of $X$ via observing $Z$. The Fano's inequality reveals the inherent difficulty of this task:

\begin{lemma}[Fano's inequality \cite{Fan61}]\label{lem:fano}
	For any function $f:W\to U$, it holds that
	\[
		\Pr{f(Z)\ne X} \ge \frac{H(X)-I(X;Z)-\log 2}{\log\abs{U}}\,.
	\]
\end{lemma}%\shuai{what is the relationship of $f$ with $X$?}

If we assume $Y=(Y_1,\dots,Y_n)$ is a vector of random variables such that $\set{Y_i}_{i\in [n]}$ are mutually independent conditional on $X$, then we have the following lemma of tensorization of mutual information:

\begin{lemma}[Tensorization of Mutual Information]\label{lem:tensor}
If $Y=(Y_1,\dots,Y_n)$ and random variables $\set{Y_i}_{i\in[n]}$ are mutually independent conditional on $X$, then
\[
	I(X;Y)\le \sum_{i=1}^n I(X;Y_i)\,.
\]
\end{lemma}
\begin{proof}
	By the chain rule of the mutual information, 
	\[
		I(X;Y) = \sum_{i=1}^n I(X;Y_i|Y_1,\dots,Y_{i-1})\,.
	\]
For every $i\in [n]$, we have
\[
	I(X;Y_i|Y_1,\dots,Y_{i-1}) = H(Y_i|Y_1,\dots,Y_{i-1}) - H(Y_i|X,Y_1,\dots,Y_{i-1}) \le H(Y_i) - H(Y_i|X)=I(X;Y_i)\,,
\]
where we use the fact that $H(Y_i|X,Y_1,\dots,Y_{i-1})=H(Y_i|X)$ due to the conditional mutual independence.
\end{proof}

\section{The Algorithm}\label{sec:alg}

\subsection{Online Stochastic Mirror Descent (OSMD)}\label{sec:osmd}

Our algorithm is based on the \emph{Online Stochastic Mirror Descent} framework that has been widely used for bandit problems in various settings. Assuming the set of arms is $[n]$, possibly with many additional structures, a typical OSMD algorithm usually consists of following steps:% \shuai{also capitalize of the bullets in section 2.2?}
\begin{itemize}
	\item Pick some initial distribution $X_1$ over all $n$ arms.
	\item For each round $t=1,2,\dots,T$:
		\begin{itemize}
			\item Tweak $X_t$ according to the problem structure to get a distribution $\tilde X_t$ over $n$ arms.
			\item The adversary chooses some (unknown) loss vector $\ell_t:[n]\to [0,1]$ with the knowledge of all previous information including $\tilde X_t$. The algorithm then picks an arm $A_t\sim \tilde X_t$ and pays a loss $\ell_t(A_t)$. After this, the algorithm observes some (partial) information $\Phi_t$ about $\ell_t$.
			\item Construct an estimator $\hat \ell_t$ of $\ell_t$ using collected information $\Phi_t$, $A_t$ and $\tilde X_t$.
			\item Compute an updated distribution $X_{t+1}$ from $X_t$ using mirror descent with a pre-specified potential function $\Psi$ and the estimated ``gradient'' $\hat\ell_t$.
		\end{itemize}
\end{itemize}

Although the framework of OSMD is standard, there are several key ingredients left for the algorithm designer to specify:
\begin{itemize}
	\item How to construct the distribution $\tilde X_t$ from $X_t$?
	\item How to construct the estimator $\hat \ell_t$?
	\item How to pick the potential function $\Psi$?
\end{itemize}
In general, filling the blanks above heavily relies on the problem structure and sometimes requires ingenious construction to achieve low regret. We will first describe our algorithm and then explain our choices in \Cref{sec:algo}.

\subsection{The Algorithm for Bandits with Graph Feedback}\label{sec:algo}

Let $G=(V,E)$ be a directed graph where $V=[n]$. In the whole section, $G$ is the input instance to the problem.

Some offline preprocessing is required before the online part of the algorithm. We first solve the linear program $\mathscr{P}$ to get an optimal solution $\set{x^*_i}_{i\in [n]}$. Recall $\delta^*(G)=\sum_{i\in[n]}x_i^*$ is the fractional weak domination number of $G$. Define a distribution $\*u\in \Delta_n$ by normalizing $\set{x^*_i}_{i\in [n]}$, i.e., we let $\*u(i)=\frac{x^*_i}{\sum_{j\in[n]}x^*_j}$ for all $i\in[n]$. The distribution $\*u$ will be the \emph{exploration distribution} whose function will be explained later. Define parameters $\gamma=\tp{\frac{\delta^*(G)\log n}{T}}^{1/3}$, $\eta=\frac{\gamma^2}{\delta^*(G)}$ where $\gamma$ is the \emph{exploration rate} and $\eta$ is the step size in the mirror descent. %\shuai{usually in bandit writing, these are input parameters and the result holds for some values. but I find it may not be easy to change under this flow. also this would need $T$. sometimes people would discuss how to remove the knowledge of $T$}
	Finally, we let the potential function $\Psi:\mathbb{R}_{\ge 0}^n\to\mathbb{R}$ be $\*x\in\mathbb{R}_{\ge 0}^n\mapsto \sum_{i=1}^n \*x(i)\log\*x(i)$ with the convention that $0\cdot \log 0 = 0$. When restricted to $\Delta_n$, $\Psi(\*x)$ is the negative entropy of the distribution $\*x$.

\begin{algorithm}[!h]
	\Begin{
		$X_1\gets\argmin\limits_{a\in \Delta_n}\Psi(a)$\;
	\For{$t=1,2,\dots,T$}{
		$\tilde{X}_t\gets (1-\gamma)\cdot X_t+\gamma\cdot \*u$\;
		\tcc{use $\*u$ to explore with rate $\gamma$.}
		Play $A_t\sim \tilde X_t$ and observe $\ell_t(j)$ for all $j\in N_{\mathrm{out}}(A_t)$\;		
		  \tcc{If $j\not\in N_{\mathrm{out}}(A_t)$, the value of $\ell_t(j)$ is unset.}
		$\forall j\in [n]:\;\hat{\ell}_t(j)\gets \frac{\ind{j\in N_{\mathrm{out}}(A_t)}}{\sum_{i\in N_{\mathrm{in}}(j)}\tilde X_t(i)}\cdot\ell_t(j)$\; 
		\tcc{For $j\not\in N_{\mathrm{out}}(A_t)$, $\hat\ell_t(j)=0$.}  
		$X_{t+1}\gets\argmin\limits_{x\in\Delta_n} \eta\inner{x}{\hat{\ell}_t}+B_\Psi(x,X_t)$\; 		
		\tcc{The update rule of mirror descent w.r.t.\ $\Psi$.}
		}
	}
	\caption{Online Stochastic Mirror Descent with Exploration}
	\label[Algorithm]{alg:osmde}
\end{algorithm}

Clearly our algorithm implements OSMD framework by specializing the three ingredients mentioned in \Cref{sec:osmd}. 
\begin{itemize}
	\item We choose $\tilde X_t=(1-\gamma)\cdot X_t+\gamma\cdot \*u$. This means that our algorithm basically follows $X_t$ but with a certain probability $\gamma$ to explore the arms according to $\*u$. The reason for doing this is to guarantee that each arm has some not-so-small chance to be observed. It will be clear from the analysis of OSMD that the performance of the algorithm depends on the variance of $\hat \ell_t$, and a lower bound for each $\tilde X_t(i)$ implies an upper bound on the variance. On the other hand, we cannot choose $\gamma$ too large since it is $X_t$ who contains information on which arm is good, and our probability to follow $X_t$ is $1-\gamma$. Therefore, our choice of $\gamma$ is optimized with respect to the trade-off between the two effects. The Exp3.G algorithm in~\cite{ACDK15} used a uniform distribution over the weakly dominating set as an exploration probability instead of $\*u$, which is the only difference between the two algorithms and leads to different graph parameters in regret bounds. Moreover, our exploration probability can be efficiently computed by solving the linear program~$\mathscr{P}$ while it is NP-hard to determine theirs.
	\item Our estimator $\hat\ell_t$ is a simple unbiased estimator for $\ell_t$, namely $\E{\hat\ell_t}=\ell_t$.
	\item The potential function we used is the negative entropy function. In fact, other choices of potential functions also work for our purpose and achieve the same regret bound. We choose it because it is simple and performs best among common choices. The OSMD algorithm with negative entropy function coincides with previous algorithms for the problem in the literature (e.g.\ \cite{freund1997decision,auer2002nonstochastic,radlinski2008learning}) where it was called EXP3 in disguise. %We will \emph{not} justify that our $\Psi$ is the best for the problem.
\end{itemize}

%The main theorem of this section is 
%\begin{theorem}\label{thm:regret}
%For any $T\geq n^3\log(n)/{\delta^*}^2(G)$, the regret of \Cref{alg:osmde} is $O\tp{\tp{\delta^*(G)\log n}^{\frac{1}{3}}T^{\frac{2}{3}}}$.
%\end{theorem}

We will analyze the regret of~Algorithm \ref{alg:osmde} and prove \Cref{thm:regret} in \Cref{sec:ub-proof}.

\section{The Analysis}\label{sec:ub-proof}

Since our algorithm only deviates from the standard OSMD algorithm by incorporating an additional exploration term $\gamma\cdot \*u$, the regret naturally consists of two parts: the standard OSMD regret and the amount introduced by the additional exploration.

Fix a sequence of loss function $\ell_1,\dots,\ell_T$ and let $a^*=\argmin_{a\in[n]}\sum_{t=1}^T\ell_{t}(a)$ be the optimal arm.
\begin{lemma}\label{lem:gamma} For any time horizon $T\in\^N$, the Algorithm~\ref{alg:osmde} satisfies
	\[
		R(G,T)\le \sum_{t=1}^T\E{\inner{X_t-\*e_{a^*}}{\hat\ell_t}}+\gamma T\,.
	\]	
\end{lemma}
\begin{proof}
	For every $t=1,2,\dots,T$, let $\+F_t$ be the $\sigma$-algebra generated by the random variables appeared at and before the $t$-th round. Define 
	$\E[t]{\cdot}=\E{\cdot\vert\mathcal{F}_t}$, then 
\[
		R(G,T)=\E{\sum_{t=1}^T\ell_t(A_t)}-\sum_{t=1}^T\ell_t(a^*) = \sum_{t=1}^T\E{\E[t-1]{\ell_t(A_t)}} - \sum_{t=1}^T\E{\ell_t(a^*)}.
\]
%\todo{abusing notation here! $\ell_t$ is either a function of a vector!}
Since $\tilde X_t$ is $\+F_{t-1}$-measurable and $A_t\sim\tilde X_t$, we have 
\begin{align*}
	R(G,T)&=\sum_{t=1}^T\E{\E[t-1]{\ell_t(A_t)}} - \sum_{t=1}^T\E{\ell_t(a^*)}\\
	&=\sum_{t=1}^T\E{\inner{\tilde{X}_t}{\ell_t}} - \sum_{t=1}^T\inner{\*e_{a^*}}{\ell_t}\le \sum_{t=1}^T\E{\inner{X_t+\gamma\cdot\*u}{\ell_t}} - \sum_{t=1}^T\inner{\*e_{a^*}}{\ell_t}\\
	&=\sum_{t=1}^T\E{\inner{X_t-\*e_{a^*}}{\ell_t}}+\sum_{t=1}^T\gamma\cdot\inner{\*u}{\ell_t}\le \sum_{t=1}^T\E{\inner{X_t-\*e_{a^*}}{\hat{\ell}_t}}+\gamma\cdot T\,,
\end{align*}
where in the last inequality we used the facts that $\hat\ell_t$ is an unbiased estimator for $\ell_t$ and $\inner{\*u}{\ell_t}\le 1$.
\end{proof}

Expanding the first term, we have the following result.

\begin{lemma}\label{lem:osmde}
	\[
	R(G,T)\leq\frac{D_\Psi(\Delta_n)}{\eta}+\frac{\eta}{2}\sum_{t=1}^T\E[A_t\sim\tilde{X}_t]{\sup_{\*z\in[Y_t,X_t]} \norm{\hat\ell_t}_{(\nabla^2\Psi(\*z))^{-1}}^2}+\gamma T\,,
	\]
	where $Y_t=\argmin_{\*y\in\interior(\dom(\Psi))}\tp{\eta\inner{\*y}{\hat\ell_t}+B_\Psi(\*y,X_t)}$.
\end{lemma}%\shuai{the domain is the vector set of all non-negative entries?}\shuai{later we have $Y_t \le X_t$, so perhaps we directly write the interval as $[Y_t, X_t]$}
%\shuai{later we use $\*z$ but here $z$}

	\Cref{lem:osmde} is a consequence of \Cref{lem:gamma} and an upper bound for $\sum_{t=1}^T\E{\inner{X_t-\*e_a}{\hat{\ell}_t}}$. The latter is a standard regret bound for OSMD and a proof can be found in, e.g.\ ~\cite{ZL19} and references therein. %We include a proof of \Cref{lem:osmde} for the sake of completeness in the appendix.

We are now ready to prove \Cref{thm:regret}. It is proved by plugging our choice of potential function into the bound of \Cref{lem:osmde}.

\begin{proof}[Proof of \Cref{thm:regret}]
	Recall that we choose $\Psi(\*x)=\sum_{i\in[n]}\*x(i)\log\*x(i)$ for every $\*x\in\Delta_n$. 
	Direct calculation yields $D_\Psi(\Delta_n)\le \log n$, $(\nabla^2\Psi(\*z))^{-1}=\text{diag}(\*z(1),\dots,\*z(n))$ and $Y_t(i)=X_t(i)\cdot e^{-\eta\hat\ell_t(i)}\le X_t(i)$ for all $t$ and $i\in [n]$. 	Therefore
	\begin{align*}
		\E[A_t\sim \tilde{X}_t]{\sup_{\*z\in[{X}_t,Y_t]} \norm{\hat\ell_t}_{(\nabla^2\psi(\*z))^{-1}}^2}
		&=\E[A_t\sim \tilde{X}_t]{\sup_{\*z\in[{X}_t,Y_t]}\sum_{i=1}^n\frac{\ind{i\in N_{\text{out}}(A_t)}^2}{\tp{\sum_{j\in N_{\mathrm{in}}(i)}\tilde X_t(j)}^2}\cdot\ell_{t}(i)^2 \cdot \*z(i)}\\
		&\le \E{\sum_{i=1}^n\frac{X_{t}(i)}{\sum_{j\in N_{\mathrm{in}}(i)}\tilde X_t(j)}}.
	\end{align*}
	It remains to lower bound $\sum_{j\in N_{\mathrm{in}}(i)}\tilde X_t(j)$ for every $i\in[n]$, which is the probability that the arm $i$ is observed at the $t$-th round. We require that the probability is not too small compared to $X_t(i)$ for every $i\in[n]$. Recall $U=\{i\notin N_{\text{in}}(i)\}$ denotes the set of vertices without self-loops. Then for every $i\not\in U$, the self-loop on $i$ guarantees that the chance for $i$ to be observed is comparable to $X_t(i)$. On the other hand, if $i\in U$, we use our additional exploration term $\gamma\cdot\*u$ to lower bound the probability.
	
	    It is clear that $\gamma\leq\frac{1}{2}$ by our choice of $\gamma$ and $T$. So the contribution of vertices in $V\setminus U$ is 
	 \begin{equation}\label{eqn:V-U}
		\sum_{i\notin U}\frac{X_t(i)}{\sum_{j\in N_{\mathrm{in}}(i)}\tilde X_t(j)}=\sum_{i\notin U}\frac{X_t(i)}{\sum_{j\in N_{\text{in}}(i)}\tp{(1-\gamma)\cdot X_{t}(j)+\gamma\cdot \*u(j)}}\leq\sum_{i\notin U}\frac{1}{1-\gamma}\leq 2n\,.
	 \end{equation}
	The contribution of vertices in $U$ is
	\begin{equation}\label{eqn:U}
	\sum_{i\in U}\frac{X_t(i)}{\sum_{j\in N_{\mathrm{in}}(i)}\tilde X_t(j)}\leq \sum_{i\in U}\frac{X_t(i)}{\gamma\sum_{j\in N_{\text{in}}(i)}\*u(j)}\stackrel{(\heartsuit)}{\leq}\frac{\sum_{i\in U}X_{t}(i)\cdot \delta^*(G)}{\gamma}\leq\frac{\delta^*(G)}{\gamma}\,,
	\end{equation}
	where $(\heartsuit)$ is due to the first constraint of the linear program $\mathscr{P}$ and our definition of $\*u$.
		
	Combining~\eqref{eqn:V-U} and~\eqref{eqn:U}, we obtain
	\begin{equation}\label{eqn:regret}
		R(G,T)\leq\frac{\log n}{\eta}+\eta nT+\frac{\eta\delta^*(G)T}{2\gamma}+\gamma T\,.
	\end{equation}

	The theorem follows by plugging in our parameters $\gamma=\tp{\frac{\delta^*(G)\log n}{T}}^{1/3}$, $\eta=\frac{\gamma^2}{\delta^*(G)}$ and assuming $T\geq n^3\log n/{\delta^*}(G)^2$.
\end{proof}

\section{Lower Bounds}\label{sec:lb}

In this section, we prove several lower bounds for the regret in terms of different graph parameters. All the lower bounds obtained in this section are based on a \emph{meta lower bound} (\Cref{thm:meta-lb}) via the notion of \emph{$k$-packing independent set}.

\begin{definition}
	Let $G=(V,E)$ be a directed graph and $k\in\^N$. A set $S\subseteq V$ is a \emph{$k$-packing independent set} of $G$ if
	\begin{itemize}
		\item $S$ is an independent set;
		\item for any $v\in V$, it holds that $\abs{N_{\-{out}}(v)\cap S}\le k$.
	\end{itemize}
\end{definition}

Intuitively, if a graph contains a large $k$-packing independent set $S$, then one can construct a hard instance as follows:
\begin{itemize}
	\item All arms in $V-S$ are bad, say with loss $1$;
	\item All arms in $S$ have loss $\-{Ber}\tp{\frac{1}{2}}$ except a special one with loss $\-{Ber}\tp{\frac{1}{2}-\eps}$.
\end{itemize}
Then any algorithm with low regret must successfully identify the special arm from $S$ without observing arms in $S$ much (since each observation of arms in $S$ comes from pulling $V-S$, which costs a penalty at least $\frac{1}{2}$ in the regret), and we can tweak the parameters to make this impossible. A similar idea already appeared in \cite{ACDK15}. However, we will formally identify the problem of minimizing regret on this family of instances with the problem of \emph{best arm identification}. Therefore, stronger lower bounds can be obtained using tools from information theory.

\begin{theorem}[Restate of \Cref{thm:meta-lb}]%\label{thm:meta-lb}
	Let $G=(V,E)$ be a directed graph. If $G$ contains a $k$-packing independent set $S$ with $\abs{S}\ge 2$, then for any randomized algorithm and any time horizon $T$, there exists a sequence of loss functions such that the expected regret is $\Omega\tp{\tp{\max\set{\frac{|S|}{k},\log |S|}}^{\frac{1}{3}}\cdot T^{\frac{2}{3}}}$. 
\end{theorem}

\begin{remark}
	If the maximum independent set of the graph $G$ is of size one and $G$ is ``weakly observable'', then it has been shown in~\cite{ACDK15} that $R(G,T)=\Omega\tp{T^{\frac{2}{3}}}$ for any algorithm. If the graph has no independent set, which means each vertex contains a self-loop, then the graph is ``strongly observable'' and its regret can be $O(T^{\frac{1}{2}})$. In particular, the problem of vanilla $n$-armed bandits falls into this case.
\end{remark}

We delay the proof of \Cref{thm:meta-lb} to \Cref{sec:meta-lb} and discuss some of its consequences in the remaining of this section. We recover and strengthen the lower bound based on the (integral) domination number of \cite{ACDK15} in \Cref{sec:alon}. Then we prove \Cref{thm:dual-lb} in \Cref{sec:dual-lb} and discuss some of its useful corollaries. Finally, we discuss in \Cref{sec:instance-lb} when our lower bounds are optimal.

Our main technical contribution here is that we relate $\delta^*$ to the lower bound as
well. This is achieved via applying the strong duality theorem of linear programming and using a new rounding
method to construct hard instances from fractional solutions of the dual linear programming. This approach towards
the lower bounds is much cleaner and in many cases stronger than previous ones in \cite{ACDK15}. To the best of our knowledge, the method is new to the community of bandits algorithms

\subsection{Lower Bound via the (Integral) Weak Domination Number}\label{sec:alon}

We first use \Cref{thm:meta-lb} to recover and strengthen lower bounds in \cite{ACDK15}. Let $G=(V,E)$ be a directed graph and $U\subseteq V$ be the set of vertices without self-loops. 	

The weakly dominating set of $U$ is a set $S\subseteq V$ such that for every $u\in U$, there exists some $v\in S$ with $(v,u)\in E$.  The weak domination number, denoted by $\delta(G)$, is the size of the minimum weakly dominating set of $U$. In fact, $\delta(G)$ is the optimum of the integral restriction of the linear program $\mathscr{P}$ in \Cref{sec:lp}:
\begin{center}
	\begin{minipage}{0.6\textwidth}
		\begin{alignat}{3}
			\notag &\text{minimize} & \sum_{i\in V} x_i &\\
			\tag{$\mathscr{P}'$}&\text{subject to} \quad &\sum_{i\in N_{\mathrm{in}}(j)}  x_i\ge 1, &\quad \forall j\in U\,,\\
			\notag& &\quad x_i\in\set{0,1}, &\quad \forall i\in V\,.
		\end{alignat}
	\end{minipage}
\end{center}

\smallskip

The following structural lemma was proved in \cite{ACDK15}
\begin{lemma}
\label{lem:alonpacking}	The graph $G$ contains a $(\log n)$-packing independent set $S$ of size at least $\frac{\delta(G)}{50\log n}$.
\end{lemma}
Applying \Cref{lem:alonpacking} to \Cref{thm:meta-lb}, we obtain 

\begin{theorem}\label{thm:strong-alon} 
	If $G$ is weakly observable, then for any algorithm and any sufficiently large time horizon $T\in\^N$, there exists a sequence of loss functions such that
	\[
	R(G,T)=\Omega\tp{\max\set{\frac{\delta(G)}{\log^2 n},\log\tp{\frac{\delta(G)}{\log n}}}^{1/3}\cdot T^{2/3}}\,.
	\]
 \end{theorem}%\shuai{we can remove the bracket out of max}
 
 Note that the bound in~\cite{ACDK15} is $R(G,T)=\Omega\tp{\max\set{\frac{\delta(G)}{\log^2 n},1}^{1/3}\cdot T^{2/3}}$. \Cref{thm:strong-alon} outperforms when $\omega(\log n)<\delta(G)<o(\log^2n\cdot \log\log n)$.
 %\chihao{I think for $n<n_0$, everything is constant and can be hidden by $\Omega$?}
 %\shuai{yes. I mean where do we need 'non-decreasing'}

\subsection{Lower Bound via the Linear Program Dual}\label{sec:dual-lb}

In this section, we use \Cref{thm:meta-lb} to derive lower bounds in terms of $\delta^*(G)$. Recall the linear program $\mathscr{D}$ in \Cref{sec:lp} whose optimum is $\zeta^*(G) = \delta^*(G)$ by the strong duality theorem of linear programming. Consider its integral restriction $\mathscr{D}'$:

\begin{center}
	\begin{minipage}{0.6\textwidth}
		\begin{alignat}{3}
			\notag &\text{maximize} & \sum_{j\in U} y_j &\\
			\tag{$\mathscr{D}'$} &\text{subject to} \quad & \sum_{j\in N_{\mathrm{out}}(i)\cap U}  y_j&\le 1, &\quad \forall i\in V\\
			\notag & & y_j\in\set{0,1}, &\quad \forall j\in U
		\end{alignat}
	\end{minipage}
\end{center}

For every feasible solution $\set{\hat y_j}_{j\in U}$ of $\mathscr{D}'$, the set $S\defeq \set{j\in U\cmid \hat y_j=1}$ is called a \emph{vertex packing set} on $U$. It enjoys the property that for every $i\in V$, $\abs{N_{\-{out}}(i)\cap S}\le 1$.

Let $\zeta(G)$ be the optimum of $\mathscr{D}'$, namely the size of the maximum vertex packing set on $U$. Let $\alpha \defeq \frac{\zeta^*(G)}{\zeta(G)}$ be the integrality gap of $\mathscr{D}$. In the following, we will write $\delta^*, \delta, \zeta^*,\zeta$ for $\delta^*(G),\delta(G),\zeta^*(G),\zeta(G)$ respectively if the graph $G$ is clear from the context.

We now prove \Cref{thm:dual-lb}, which is restated here. 

\begin{theorem}[Restate of \Cref{thm:dual-lb}] If $G$ is weakly observable, then for any algorithm and any sufficiently large time horizon $T\in\^N$, there exists a sequence of loss functions such that
	\[
		R(G,T)=\Omega\tp{\tp{\frac{\delta^*}{\alpha}}^{1/3}\cdot T^\frac{2}{3}}\,.
	\]
\end{theorem}

\begin{proof}
Since $\delta^* = \zeta^* = \alpha\cdot \zeta$, the bound in \Cref{thm:dual-lb} is equivalent to $R(G,T)=\Omega\tp{\zeta^{\frac{1}{3}}\cdot T^{\frac{2}{3}}}$. We will prove that $G$ contains a $1$-packing independent set of size $\Theta(\zeta)$, then the theorem follows from \Cref{thm:meta-lb}.

Let $\set{y_j^{\dagger}}_{j\in U}$ be the optimal solution of the integral program $\mathscr{D}'$. Let $S^\dagger\defeq \set{j\in U\cmid y_j^\dagger =1}$ be the corresponding vertex packing set. Then clearly $\zeta = \abs{S^\dagger}$. We will show that there exists a $1$-packing independent set $H\subseteq S^\dagger$ with $\abs{H}\ge \abs{S^\dagger}/3$.

We use the following greedy strategy to construct $H$.

	\begin{itemize}
		\item \textsc{Initialization:} Set $H=\emptyset$ and $S^\prime=S^\dagger$.
		\item \textsc{Update:} While $S' \neq\emptyset$: Pick a vertex $v\in S'$ with minimum $\abs{N_{\-{in}}(v)\cap S'}$;  Set  $H\gets H\cup\{v\}$; $S'\gets S'\setminus \tp{N_\-{in}(v)\cup \{v\}\cup N_\-{out}(v)}$.
	\end{itemize}
	First of all, the set $H$ constructed above must be an independent set as whenever we add some vertex $v$ into $H$, we remove all its incident vertices, both its in-neighbors and out-neighbors, from $S'$. 	Clearly each $S'$ after removing these vertices is still a vertex packing set. Therefore, every $v\in S'$ has out-degree at most one in $G[S']$. This implies that the vertex $v\in S'$ with minimum $\abs{N_\-{in}(v)\cap S'}$, or equivalently minimum in-degree in $G[S']$, satisfies $\abs{N_\-{in}(v)\cap S'}\le 1$. So the update step $S'\gets S'\setminus \tp{N_\-{in}(v)\cup \{v\}\cup N_\-{out}(v)}$ removes at most three vertices from $S'$. This concludes that $H$ is a $1$-packing independent set of size at least $\abs{S^\dagger}/3$.
\end{proof}	

\Cref{thm:dual-lb} is less informative if we do not know how large the integrality gap $\alpha$ is. On the other hand, the integrality gap of linear programs for packing programs has been well-studied in the literature of approximation algorithms. The following bound due to~\cite{Sri99} is tight for general graphs.

\begin{lemma}[\cite{Sri99}]\label{lem:general-gap}
	For any directed graph $G$, the integrality gap $\alpha=O\tp{n/\delta^*}$.
\end{lemma}

Therefore we have 

\begin{corollary}\label{cor:gap-lb}
If $G$ is weakly observable, then for any algorithm and any sufficiently large time horizon $T\in\^N$, there exists a sequence of loss functions such that\[
	R(G,T)=\Omega\tp{\tp{\frac{\delta^*}{n/\delta^*}}^{\frac{1}{3}}\cdot T^{\frac{2}{3}}}\,.
\]
\end{corollary}

The bound for the integrality gap in \Cref{lem:general-gap} holds for any graphs. It can be greatly improved for special graphs.

An interesting family of graphs with small $\alpha$ is those bounded degree graphs.  If the in-degree of every vertex in $U$ is bounded, we have the following bound for the integrality gap:

\begin{lemma}[\cite{BKNS12}]
	If the in-degree of every vertex in $U$ is bounded by a constant $\Delta$, then the integrality gap $\alpha\le 8\Delta$.
\end{lemma}

\begin{corollary}
	Let $\Delta\in\^N$ be a constant and $\+G_\Delta$ be the family of graphs with maximum in-degree $\Delta$. Then for every weakly observable $G=(V,E)\in\+G_\Delta$, any algorithm and any sufficiently large time horizon $T\in\^N$,  there exists a sequence of loss functions such that
	\[
	R(G,T)=\Omega((\delta^*)^{\frac{1}{3}}\cdot T^{\frac{2}{3}})\,.
	\]
\end{corollary}

Next, we show the integrality gap of another broad family of graphs, the $1$-degenerate graphs, is $1$. Recall that we say a directed graph $G$ is $1$-degenerate if one can iteratively apply the following operators in arbitrary orders on $G$ to get an empty graph:
	\begin{itemize}
		\item Pick a vertex with in-degree one and remove the in-edge;
		\item Pick a vertex with in-degree zero and out-degree at most one, and remove both the vertex and the out-edge.
	\end{itemize}
It is easy to verify that trees, both directed and undirected, are $1$-degenerate. We prove that

\begin{lemma}
	For any $1$-degenerate directed graph, the integrality gap $\alpha=1$.
\end{lemma}
\begin{proof}
	Let $G=(V,E)$ be a $1$-degenerate directed graph.  We show that we can construct a vertex packing set $S$ with $\abs{S}\ge \sum_{j\in U}\hat y_j$ for any feasible solution $\set{\hat y_j}_{j\in U}$ of $\mathscr{D}$. The lemma follows by applying this to the optimal solution.
		
	We use $\*y_S\in\set{0,1}^U$ to denote the indicator vector of $S$. So we have for every $i\in U$, $\*y_S(i)=1\iff i\in S$. The construction is to apply the following greedy strategy to determining $\*y_S$ until the graph is empty:
	\begin{itemize}
		\item Pick a vertex $i$ in $G$ with in-degree one. Let $(j,i)$ be the unique in-edge of $i$. Remove the edge $(j,i)$ from $E$. If $i\in U$ and the value of $\*y_S(i)$ is not determined, then (1) set $\*y_S(i)=1$; (2) for all $k\in U\setminus\set{i}$ such $(j,k)\in E$, set $y_S(k)=0$.
		\item Pick a vertex with in-degree zero and out-degree at most one, and remove both the vertex and the out-edge. Keep doing so until no such vertex exists in $G$.
	\end{itemize}
	It is clear that the above algorithm terminates at an empty $G$ since all operations to the graph coincide with ones defining $1$-degeneration. We only need to verify that 
	\begin{enumerate}
		\item $\*y_S$ is a feasible solution of $\mathscr{D}$; and
		\item $\sum_{j\in U}\*y_S(j) \ge \sum_{j\in U}\hat y_j$ for any feasible solution $\set{\hat y_j}_{j\in U}$.
	\end{enumerate}
	Let us first verify (1). The vector $\*y_S$ can become infeasible only when we set some $\*y_S(i)=1$. Note that this happens only when the in-degree of $i$ is one, or equivalent there is only a unique edge $(j,i)$ pointing to $i$. We do not violate the constraint on $j$ as we set all $y_S(k)=0$ for $k\in U\setminus\set{i}$ and $(j,k)\in E$. It is still possible that there exists some other $j'\in V$ such that the edge $(j',i)$ exists but has been removed. Since the value of $\*y_S(i)$ is not determined before, this happens only if $j'$ has out-degree one at the beginning, and so the constraint on $j'$ cannot be violated either. %\shuai{this part is a bit complicated. perhaps we can just use contradiction}	
	
	To see (2), we assume the value on each $j\in U$ is $\hat y_j$ at the beginning. Our procedure to construct $\*y_S$ can be equivalently viewed as a process to change each $\hat y_j$ to either $0$ or $1$. That is, after we set the value of some $\*y_S(j)$ to $0$ or $1$, we change $\hat y_j	$ to the same value. It is easy to verify that during the process, we never decrease $\sum_{j\in U}\hat y_j$. At last, $\*y_S(j)=\hat y_j$ for all $j\in U$ and some optimal solution $\{\hat{y}_j\}_{j\in U}$, and (2) is verified. 	
	\end{proof}
	
\begin{corollary}
	Let $G$ be a $1$-degenerate weakly observable graph. Then for any algorithm and any sufficiently large time horizon $T\in\^N$,  there exists a sequence of loss functions such that
	\[
	R(G,T)=\Omega((\delta^*)^{\frac{1}{3}}\cdot T^{\frac{2}{3}})\,.
	\]
\end{corollary}

Comparing with the upper bound in \Cref{thm:regret} ($R(G,T)=O\tp{\tp{\delta^*\log n}^{\frac{1}{3}}T^{\frac{2}{3}}}$), we conclude that our lower bound is tight up to a factor of $\tp{\log n}^{\frac{1}{3}}$ on $1$-degenerate graphs and graphs of bounded degree.

\subsection{Instances with Optimal Regret}\label{sec:instance-lb}

In this section, we will examine several families of graphs in which the optimal regret bound can be obtained using tools developed in this article.

\subsubsection{Complete bipartite graphs}

Let $G=(V_1\cup V_2,E)$ be an undirected complete bipartite graph with $n=\abs{V_1}+\abs{V_2}$. Clearly $\delta^*=\delta=2$. Therefore both our \Cref{thm:regret} and the algorithm in~\cite{ACDK15} satisfy $R(G,T)=O\tp{(\log n)^{\frac{1}{3}}\cdot T^{\frac{2}{3}}}$.

Assuming without loss of generality that $\abs{V_1}\ge \abs{V_2}$, then $V_1$ is a $\abs{V_1}$-packing independent set of size at least $\frac{n}{2}$. Therefore it follows from \Cref{thm:meta-lb} that any algorithm satisfies $R(G,T)=\Omega\tp{(\log n)^{\frac{1}{3}}\cdot T^{\frac{2}{3}}}$. Note that the lower bound in~\cite{ACDK15} is $\Omega\tp{T^{\frac{2}{3}}}$ for this instance.

\subsubsection{Orthogonal relation on $\^F_2^k$} 

Our algorithm outperforms the one in~\cite{ACDK15} when $\delta^*\ll \delta$. Let us now construct a family of graphs where $\frac{\delta}{\delta^*}=\Omega\tp{\log n}$. This construction is folklore in the literature of approximation algorithms as it witnesses that the integrality gap of the natural linear programming relaxation for \emph{set cover} is $\Omega(\log n)$.

Let $\^F_2$ be the finite field with two elements and $k\in\^N$ be sufficiently large. The vertex set of the undirected graph $G=(V_1\cup V_2, E)$ consists of two disjoint parts $V_1$ and $V_2$ where $V_1$ is isomorphic to $\^F_2^k$ and $V_2$ is isomorphic to $\^F_2^k\setminus\set{\*0}$. Therefore we can write $V_1=\set{x_\alpha}_{\alpha\in \^F_2^k}$ and $V_2=\set{y_\beta}_{\beta\in \^F_2^k\setminus\set{\*0}}$. The set of edges $E$ is as follows:
	\begin{itemize}
		\item $E$ is the edge set such that $G[V_1]$ is a clique and $G[V_2]$ is an independent set;
		\item For every $x_\alpha\in V_1$ and $y_\beta\in V_2$, $\set{x_\alpha,y_\beta}\in E$ iff $\inner{\alpha}{\beta}=\sum_{i=1}^k \alpha(i)\cdot \beta(i)=1$, where all multiplications and summations are in $\^F_2$.
	\end{itemize}
Let $n=\abs{V_1}+\abs{V_2}=2^{k+1}-1$. It is clear that the degree of each vertex $y_\beta\in V_2$ is $2^{k-1}=\frac{n+1}{4}$. A moment's reflection should convince you that $\delta^*(G)\le 2$ as we can put a fraction of $\frac{4}{n+1}$ on each $x_\alpha\in V_1$. 

Therefore it follows from \Cref{thm:regret} that our algorithm has regret $O\tp{\tp{\log n}^{\frac{1}{3}}\cdot T^{\frac{2}{3}}}$ on this family of instances. Moreover, $V_2$ is a $\frac{n+1}{4}$-packing independent set of size $\frac{n-1}{2}$. It follows from \Cref{thm:meta-lb} that any algorithm has regret $\Omega\tp{\tp{\log n}^{\frac{1}{3}}\cdot T^{\frac{2}{3}}}$. 

Finally, we remark that $\delta(G)=k=\log_2\tp{\frac{n+1}{2}}$. To see this, we first observe that the minimum dominating set of the graph must reside in $V_1$ since $G[V_2]$ is an independent set. Then we show any $S\subseteq V_1$ with $\abs{S}\le k-1$ cannot dominate all vertices in $V_2$. Assume $S=\set{x_{\alpha_1},x_{\alpha_2},\dots,x_{\alpha_{k-1}}}$. In fact, a vertex $y_\beta\in V_2$ is dominated by $S$ iff $\inner{\alpha_i}{\beta}=1$ for some $i\in [k-1]$. In other words, if we view each $\alpha_i$ as a column vector in $\^F_2^k$ and define the matrix $A=[\alpha_1,\alpha_2,\dots,\alpha_{k-1}]$, then $y_\beta$ is dominated by $S$ iff $A^{\mathtt{T}}\beta\ne \*0$. However, the dimension of $A^{\mathtt{T}}$ is at most $k-1$ and therefore by the rank-nullity theorem that its null space is of dimension at least one. This means that there exists a certain $\beta'\in V_2$ with $A^{\mathtt{T}}\beta'=\*0$. So $y_{\beta'}$ is not dominated by $S$. On the other hand, any $k$ vertices in $V_1$ indexed by $k$ linear independent vectors can dominate the whole graph. 

	The above fact implies that the upper bound and the lower bound of the regret for this instance proved in~\cite{ACDK15} based on $\delta$ are $O\tp{\tp{\log n}^{\frac{2}{3}}\cdot T^{\frac{2}{3}}}$ and $\Omega\tp{T^{\frac{2}{3}}}$ respectively. Therefore we conclude that both our new upper bound and lower bound are crucial for the optimal regret on this family of instances.

\section{Proof of the Meta Lower Bound (\Cref{thm:meta-lb})}\label{sec:meta-lb}

Our strategy to prove \Cref{thm:meta-lb} is to reduce the problem of minimizing regret to the problem of \emph{best arm identification} . We first define the problem and discuss its complexity in \Cref{sec:BAI}. Then we construct the reduction and prove \Cref{thm:meta-lb} in \Cref{sec:main-proof}.

\subsection{Best Arm Identification}\label{sec:BAI}

The problem of \emph{best arm identification} is formally defined as follows.

\bigskip
\nprob[12]{Best Arm Identification (\BAI)}{An instance of $n$ stochastic arms where the loss of the $i$-th arm follows $\Ber{p_i}$. Each pull of arm $i$ receives a loss $\sim\Ber{p_i}$ independently.
%the arm $i$ is observed $t_i$ times with a collection of samples $\set{X_j^{(i)}}_{1\le j\le t_i}$ following $\Ber{p_i}$ independently.
}{Determine the arm $i$ with minimum $p_i$ via pulling arms.}

\bigskip
Therefore, an instance of \BAI is specified by a vector $\*p=(p_1,\dots,p_n)\in [0,1]^n$. The goal is to design a strategy to find the arm $i$ with minimum $p_i$ via pulling arms. We call the arm with minimum $p_i$ the best arm. We want to minimize the number of pulls in total and the main result of this section is to provide lower bounds for this task: For some $\*p$, if the total number of pulls is below some threshold, then any algorithm cannot find the best arm with high probability.

In the following, we may abuse notations and simply say a vector $\*p\in [0,1]^n$ is an instance of \BAI. Now for every $j\in [n]$, we define an instance $\*p^{(j)}=(p^{(j)}_1,\dots,p^{(j)}_n)\in \^R^n$ as $
p^{(j)}_i=
\begin{cases}
	\frac{1}{2}-\eps, & i=j;\\
	\frac{1}{2}, & i\ne j,
\end{cases}$ for some $\eps\in (0,\frac{1}{2}]$. This is the collection of instances we will study. For the convenience of the reduction in \Cref{lem:BAI-hard-1}, we also define $\*p^{(0)}$ as an $n$-dimensional all-one vector.

There are several ways to explore the $n$ arms in order to determine the one with minimum mean. We first consider the most general strategy: In each round, the player can pick an arbitrary subset $S\subseteq [n]$ and pull the arms therein. The game proceeds for $T$ rounds and then the player needs to determine the best arm with collected information. Note that in each round, the exploring set $S$ may adaptively depend on previous information. 
	
	Now we fix such a (possibly randomized) strategy and denote it by $\mathscr{B}$. For every $j\in [n]$ and $i\in [n]$, we use $N_i^{(j)}$ to denote the number of times that the $i$-th arm is pulled when we run $\mathscr{B}$ on the instance $\*p^{(j)}$. Let $N^{(j)}=\sum_{i\in[n]}N_i^{(j)}$. Note that all these numbers can be random variables where the randomness comes from coins tossed in $\mathscr{B}$.

	\begin{lemma}\label{lem:BAI-hard-1}
		Assume $\eps<0.125$ and $n$ is sufficiently large. If for every $j\in[n]$, the algorithm $\mathscr{B}$ can correctly identify the best arm in $\*p^{(j)}$ with probability at least $0.999$ or outputs any arm for $\*p^{(0)}$, then for some $j\in \{0,1,\dots,n\}$,
		\[
		 \E{N^{(j)}} \ge \frac{Cn}{\eps^2}\,,
		\]
	where $C>0$ is a universal constant.
	\end{lemma}

	Our proof of \Cref{lem:BAI-hard-1} is based on a reduction from a similar problem studied in~\cite{MT04}, in which the following instances of \BAI have been considered:

\begin{itemize}
	\item $\*q^{(0)}=(q_0^{(0)},\dots,q_n^{(0)})\in \^R^{n+1}$ where for every $i\in\{0,1,\dots,n\}$, $q_i^{(0)}=
	\begin{cases}
		\frac{1}{2}-\eps, & i=0;\\
		\frac1 2, & i\ne 0.
	\end{cases}$
	\item $\forall j\in[n]:$ $\*q^{(j)}=(q_0^{(j)},\dots,q_n^{(j)})\in \^R^{n+1}$ where for every $i\in\{0,1,\dots,n\}$, $q_i^{(j)}=
	\begin{cases}
		\frac{1}{2}-\frac{\eps}{2}, & i=0;\\
		\frac{1}{2}-\eps, & i=j;\\
		\frac{1}{2}, & \mbox{otherwise.}
	\end{cases}$
\end{itemize}	
Let us fix a strategy $\mathscr{B}'$ for this collection of instances. Similarly define quantities $N_i^{(j)\prime}$ and $N^{(j)\prime}$ for $i,j=0,1\dots,n$ as we did for $\mathscr{B}$ above. The proof of~\cite[Theorem 1]{MT04} implicitly established the following:

\begin{lemma} \label{lem:oracle}
Assume $\eps<0.125$. If for every $j=0,1\dots,n$, the algorithm $\mathscr{B}'$ can correctly identify the best arm in $\*q^{(j)}$ with probability at least $0.996$, then 
	\[
	\E{N^{(0)\prime}}\ge \frac{C'n}{\eps^2}\,,
	\]
	where $C'$ is a universal constant.
\end{lemma}

Armed with \Cref{lem:oracle}, we now prove \Cref{lem:BAI-hard-1}.

\begin{proof}[Proof of \Cref{lem:BAI-hard-1}]
	Assuming for the sake of contradiction that \Cref{lem:BAI-hard-1} does not hold, we now describe an algorithm $\mathscr{B}'$ who can correctly identify the best arm in $\*q^{(j)}$ with probability at least $0.996$ for every $j\in\{0,\dots, n\}$ and $\E{N^{(j)\prime}}< \frac{C'n}{\eps^2}$ for sufficiently large $n$.
	
	Since \Cref{lem:BAI-hard-1} is false, we have a promised good algorithm $\mathscr{B}$ with $C=\frac{C'}{10}$. Given any instance $\*q^{(j)}$ with $j\in \{0,\dots,n\}$, we first use $\mathscr{B}$ to identify the best arm $i^*$ among arms in $\set{1,2,\dots,n}$. This step succeeds with probability $0.999$. Then we are left to compare arm $i^*$ with arm $0$. We pull each of the two for $K$ times and output the one with minimum practical mean. By Chernoff bound, this approach can successfully identify the best of the two with probability $0.999$ when $K=\frac{C''}{\eps^2}$ for some universal constant $C''>0$. Therefore we have $\E{N^{(j)\prime}}<\frac{Cn}{\eps^2}+\frac{C''}{\eps^2}\le \frac{C'n}{\eps^2}$ for sufficiently large $n$ and we can identify the best arm with probability at least $0.998>0.996$ by the union bound. 
\end{proof}

\Cref{lem:BAI-hard-1} is quite general in the sense that it applies to any algorithm for \BAI. If we restrict our algorithm for \BAI to some special family of strategies, then a stronger lower bound can be obtained.

Consider the following algorithm $\mathscr{C}$: In every round, the player pulls each arm once. After $T$ rounds (so each arm has been pulled $T$ times), the player determines the best arm via the collected information. Note that we do not restrict how the player determines the best arm after collecting information for $T$ rounds, his/her strategy can be either deterministic or randomized. We prove a lower bound for $T$:

\begin{lemma}\label{lem:BAI-hard-2}
	If for every $j\in [n]$, the algorithm $\mathscr{C}$ can correctly identify the best arm in $\*p^{(j)}$ with probability at least $0.5$,  then
	\[
		T\ge \frac{\log (n/4)}{16\eps^2}\,.
	\]
\end{lemma}

Note that if we apply \Cref{lem:BAI-hard-1} to $\mathscr{C}$, we can only get $T=\Omega\tp{1/\eps^2}$. The reason that we can obtain a stronger lower bound is the non-adaptive nature of $\mathscr{C}$. 

\begin{proof}[Proof of \Cref{lem:BAI-hard-2}]
	As a randomized algorithm can be viewed as a distribution of deterministic ones, we only need to prove the lower bound for deterministic algorithms. We prove the contrapositive of the lemma for a deterministic $\mathscr{C}$. Assume $T<\frac{\log (n/4)}{16\eps^2}$. We let $W=\tp{w_{ij}}_{\substack{i\in [n]\\j\in [T]}}\in [0,1]^{n\times T}$ be a random matrix where $w_{ij}$ is the loss of the $i$-th arm during the $j$-th pull. Our task is to study the following stochastic process, which is called \emph{hypothesis testing} in statistics:
		\begin{itemize}
			\item Pick $J\in [n]$ uniformly at random.
			\item Use $\*p^{(J)}$ to generate the matrix $W$.
			\item Apply $\mathscr{C}$ on $W$ to obtain $\hat J=\mathscr{C}(W)$.
		\end{itemize}
	It then follows from Fano's inequality (\Cref{lem:fano}) that
	\begin{equation}\label{eqn:fano}
		\Pr{\hat J\ne J}\ge \frac{H(J)-I(J;W)-\log 2}{\log n} = 1-\frac{I(J;W)+\log 2}{\log n}\,.
	\end{equation}

	It remains to upper bound $I(J;W)$. To ease the presentation, we write $W=\left[w^{(1)},w^{(2)},\dots,w^{(T)}\right]$  where each $w^{(j)}=(w_1^{(j)},w_2^{(j)},\dots,w_n^{(j)})^{\mathtt{T}}$ for $j\in [T]$ is an $n$-dimensional column vector. It is clear that these column vectors are mutually independent \emph{conditional on $J$}. Moreover, for each $j\in [T]$, entries in $w^{(j)}$ are mutually independent conditional on $J$ as well. Therefore, it follows from \Cref{lem:tensor} that
	\begin{equation}\label{eqn:I-bound}
		I(J;W)\le \sum_{j\in[T]} I(J;w^{(j)}) \le \sum_{j\in [T]}\sum_{i\in [n]}I(J;w_{i}^{(j)}) = nT\cdot I(J;w_1^{(1)}) \le 8\eps^2T,
	\end{equation}
	where the last inequality is from a direct calculation of $I(J;w_1^{(1)})$.
	
	Combining \Cref{eqn:fano} and \Cref{eqn:I-bound}, we obtain
	\[
	\Pr{\hat J\ne J}\ge 1-\frac{8\eps^2T+\log 2}{\log n}>\frac{1}{2}\,,
	\]
	which is a contradiction.
\end{proof}

\subsection{From \BAI to Regret}\label{sec:main-proof}

Let $G=(V,E)$ be a directed graph containing a $k$-packing independent set $S$ with $\abs{S}\ge 2$. We assume without loss of generality that $S=\set{1,2,\dots,|S|}$. We construct a collection of \emph{stochastic bandit} instances $\set{I^{(j)}}_{j\in[|S|]}$ with feedback graph $G$ as follows: For every $t\in [T]$ and $j\in [|S|]$, we use $\ell^{(j)}_t$ to denote the loss function of $I^{(j)}$ at round $t$. Let $\eps\in (0,1)$ be a parameter. Then
	\begin{itemize}
		\item For all $i\not\in S$, $\ell_t^{(j)}(i)=1$;
		\item For $i=j$, $\ell_t^{(j)}(i)\sim\Ber{\frac{1}{2}-\eps}$;
		\item For all $i\in S\setminus\set{j}$, $\ell_t^{(j)}(i)\sim\Ber{\frac{1}{2}}$.
	\end{itemize}

Given an algorithm $\mathscr{A}$, a time horizon $T$ and $j\in[n]$, we use $R(\mathscr{A},j,T)$ to denote the expected regret after $T$ rounds when applying $\mathscr{A}$ on $I^{(j)}$.

We show that for any $j\in [|S|]$, if an algorithm $\mathscr{A}$ has low expected regret on $I^{(j)}$, then one can turn it into another algorithm $\mathscr{\widehat A}$ who can identify the best arm $j$ among $S$ without exploring $S$ much. 

\begin{lemma}\label{lem:regret2BAI}
	Let $T\in \^N$ be the time horizon, $\delta\in (0,1)$ be a parameter. Let $\ol{C}=\min\set{C,1}$ where $C$ is the constant in \Cref{lem:BAI-hard-1}. Assuming there exists an algorithm $\mathscr{A}$ such that $R(\mathscr{A},j,T)\le \frac{\ol{C}\eps\delta T}{4} $ for every $j\in[|S|]$, then there exists an algorithm $\mathscr{\widehat A}$ satisfying for every $j\in [|S|]$, if we apply $\mathscr{\widehat A}$ on $I^{(j)}$ for $T$ rounds, then
	\begin{itemize}
		\item $\mathscr{\widehat A}$ output $j$ with probability at least $1-\delta$;
		\item arms in $V\setminus S$ are pulled at most $\frac{\ol{C}\eps\delta T}{2}$ times in total.
	\end{itemize}
\end{lemma}
\begin{proof}
	The algorithm $\mathscr{\widehat A}$ simply simulates $\mathscr{A}$ for $T$ rounds and outputs the mostly-pulled arm, breaking ties arbitrarily. We verify the two properties of $\mathscr{\widehat A}$ respectively.
	\begin{itemize}
		\item If the mostly-pulled arm is not $j$, then suboptimal arms must be pulled at least $\frac{T}{2}$ times. These pulls contribute at least $\frac{\eps \delta T}{2}>\frac{\ol{C}\eps\delta T}{4}$ expected regret.
		\item If arms in $V\setminus S$ are pulled more than  $\frac{\ol{C}\eps\delta T}{2}$ times in total, then these pulls already contribute more than $\frac{\ol{C}\eps\delta T}{4}$ regret.
	\end{itemize}
\end{proof}

Using this reduction, we can prove \Cref{thm:meta-lb} via lower bounds for \BAI.

\begin{proof}[Proof of \Cref{thm:meta-lb}]
	Assume both $|S|$ and $T$ are sufficiently large.
	
	We first establish the $\Omega\tp{\tp{\frac{|S|}{k}}^{\frac{1}{3}}\cdot T^{\frac{2}{3}}}$ lower bound. Choose $\eps=\tp{\frac{\abs{S}}{kT}}^{\frac{1}{3}}$ and $\delta=0.001$. Suppose there exists an algorithm $\mathscr{A}$  such that $R(\mathscr{A},j,T)\le \frac{\ol{C}}{4000}\tp{\frac{|S|}{k}}^{\frac{1}{3}}\cdot T^{\frac{2}{3}}$ for every $j\in [|S|]$. Then by \Cref{lem:regret2BAI}, we can find an algorithm $\mathscr{\widehat A}$ that can correctly identify the best arm with probability at least $0.999$, and observe arms in $S$ for at most $\frac{\ol{C}}{2000}|S|^{\frac{1}{3}}k^{\frac{2}{3}}\cdot T^{\frac{2}{3}}$ times (Since each pull of $V\setminus S$ observes at most $k$ arms in $S$).  This contradicts \Cref{lem:BAI-hard-1}.
	
	Then we establish the $\Omega\tp{\tp{\log |S|}^{\frac{1}{3}}\cdot T^{\frac{2}{3}}}$ lower bound, which needs more effort. Choose $\eps=\tp{\frac{\log |S|}{T}}^{\frac{1}{3}}$ and $\delta=0.001$. Similar to above, suppose there exists an algorithm $\mathscr{A}$  such that $R(\mathscr{A},j,T)\le \frac{\ol{C}}{4000}\tp{\log|S|}^{\frac{1}{3}}\cdot T^{\frac{2}{3}}$ for every $j\in [|S|]$. Then by \Cref{lem:regret2BAI}, we can find an algorithm $\mathscr{\widehat A}$ that can correctly identify the best arm with probability at least $0.999$, and pull arms in $V\setminus S$ for at most $\frac{\ol{C}}{2000}\tp{\log|S|}^{\frac{1}{3}}\cdot T^{\frac{2}{3}}$ times in total. 
	
		Note that each pull of an arm $v\in V\setminus S$ in $\mathscr{A}$ observes arms in a subset $S'\subseteq S$. The key observation is that we can assume without the loss of generality that $S'=S$, since this assumption only increases the power of the algorithm.  This assumption can make our algorithm non-adaptive.
		
		We rigorously prove this using a coupling argument. Assume we design an algorithm $\mathscr{\widehat B}$ in which each pull of an arm in $V\setminus S$ can observe all arms in $S$. We show that $\mathscr{\widehat B}$ can perfectly simulate $\mathscr{\widehat A}$, and therefore $\mathscr{\widehat B}$ is more powerful. So if we have a lower bound for $\mathscr{\widehat B}$, it is automatically a lower bound for $\mathscr{\widehat A}$. Note that the behavior of an algorithm for \BAI in each round is determined by the information collected and coins tossed so far. If we apply both $\mathscr{\widehat A}$ and $\mathscr{\widehat B}$ to some $I^{(j)}$ and assume (1) both algorithms use the same source of randomness; and (2) all the loss vectors $\ell_t^{(j)}$ use the same source of randomness, then the information collected by $\mathscr{\widehat A}$ at any time is a subset of $\mathscr{\widehat B}$. Therefore we can use $\mathscr{\widehat B}$ to simulate $\mathscr{\widehat A}$ and outputs the same as $\mathscr{\widehat A}$. 
	
		Based on this observation, we use the non-adaptive nature of $\mathscr{\widehat B}$ and apply \Cref{lem:BAI-hard-2} to get an immediate contradiction.
\end{proof}

\section{Time-varying Graphs}\label{sec:generalization}
 In this section, we will consider time-varying graphs. Instead of a fixed graph $G=(V,E)$ for all $T$ rounds, for each round $t\in[T]$, we have a graph $G_t=(V,E_t)$. We slightly abuse notations here and let $G=(G_t)_{t\in[T]}$ and $R(G,T)$ denote the min-max regret for time-varying graphs $G$. 
 \begin{corollary}\label{cor:tv}
 	The min-max regret for sequential time-varying graphs $G=(G_t)_{t\in[T]}$ satisfies 
 	\[
 	R(G,T)=O\tp{\tp{\overline{\delta^*}(G)\log n}^{\frac{1}{3}}T^{\frac{2}{3}}}\,,
 	\]
 	where $\overline{\delta^*}\triangleq \frac{\sum_{t=1}^T\delta_t^*(G_t)}{T}$ is the average fractional weak domination number for graphs $G$ in $T$ rounds.
 \end{corollary}
\begin{proof}
	We can slightly modify Algorithm~\ref{alg:osmde}: In each round $t$, we use $G_t$ as an input instance to calculate $\hat\ell_t$ and $\*u_t$.
Then the left proof is totally similar to the proof of \Cref{thm:regret} except for replacing \Cref{eqn:regret} with 	$Regret\leq \frac{\log n}{\eta}+\eta n T+\frac{\eta \sum_{t=1}^T\delta^{*}(G_t)}{2 \gamma}+\gamma T=\frac{\log n}{\eta}+\eta n T+\frac{\eta \bar{\delta^*} T}{2 \gamma}+\gamma T.$
\end{proof}

A similar case is the probabilistic graph model. A probabilistic graph can be denoted as a triple $\+G=(V,E,P)$ where $P:E\to (0,1]$ assigns a \emph{triggering probability} for each edge in $E$. In each round $t$,  a realization of $\+G$ is a graph $G_t=(V,E_t)$ where  $E_t=\{e\in E:O_{te}=1\}$ and here $O_{te}$ is an independent Bernoulli random variable with mean $P(e)$.  By abuse of notations, we denote by $G=(G_t)_{t\in[T]}$ a sequence of independent realizations of $\+G$. Define $R(\+G,T)=\inf_{\mathscr{A}}\sup_{\ell^*}\E{\sum_{t=1}^T \ell_t(A_t)-\ell_t(a^*)}$ as the min-max regret for the probabilistic graph $\+G$ and here the randomness comes from the algorithm and the sequential graphs $G$.
\begin{corollary}
	The min-max regret for the probabilistic graph $\+G$ satisfies
	 \[
	R(\+G,T)=O\tp{\tp{\E{\delta^*(G_1)}\log n}^{\frac{1}{3}}T^{\frac{2}{3}}}\,,
	\]
\end{corollary}
\begin{proof}
	\begin{align*}
	R(\+G,T)&=\E{\sum_{t=1}^T \ell_t(A_t)-\ell_t(a^*)}=\E{\E{\sum_{t=1}^T \ell_t(A_t)-\ell_t(a^*)\,\vert\, G}} \\
	&\stackrel{(\diamondsuit)}{=} O\tp{\E{\tp{\frac{\sum_{t=1}^{T}\delta^*(G_t)}{T}}^\frac{1}{3}}(\log n)^{\frac{1}{3}}T^{\frac{2}{3}}}
	\stackrel{(\clubsuit)}{=}O\tp{\tp{\E{\delta^*(G_1)}\log n}^{\frac{1}{3}}T^{\frac{2}{3}}},
	\end{align*}
	where $(\diamondsuit)$ follows from \Cref{cor:tv} and $(\clubsuit)$ comes from the Jensen inequality and the independence of $(G_t)_{t\in[T]}$.
\end{proof}

\section{Conclusion and Future Work}\label{sec:con}

In this article, we introduced the notions of fractional weak domination number and $k$-packing independence number respectively to prove new upper bounds and lower bounds for the regret of bandits with graph feedback. Our results implied optimal regret on several families of graphs. Although there are still some gaps in general, we believe that these two notions are the correct graph parameters to characterize the complexity of the problem.  We now list a few interesting problems worth future investigation.

\begin{itemize}
	\item Let $G$ be an $n$-vertex undirected cycle. What is the regret on this instance? \Cref{thm:regret} implies an upper bound $O\tp{(n\log n)^{\frac{1}{3}}T^{\frac{2}{3}}}$ and \Cref{thm:meta-lb} implies a lower bound $\Omega\tp{n^{\frac{1}{3}}T^{\frac{2}{3}}}$. 
	\item The lower bound $\Omega\tp{\tp{\frac{\delta^*}{\alpha}}^{\frac{1}{3}}\cdot T^{\frac{2}{3}}}$ in \Cref{thm:dual-lb} for general graphs is not satisfactory. The lower bound is proved via the construction of a $1$-packing independent set. This construction did not release the full power of \Cref{thm:meta-lb} as the lower bound in the theorem applies for any $k$-packing independent set. It is still possible to construct larger $k$-packing independent sets via rounding the linear program $\mathscr{D}$ to some ``less integral'' solution.
	\item Is \Cref{thm:meta-lb} tight? In fact, the bound for \BAI is tight since matching upper bound exists. Therefore, one needs new constructions of hard instances to improve \Cref{thm:meta-lb}, if possible.
\end{itemize}

\bibliographystyle{alpha}
\bibliography{bandits}

\newcommand{\etalchar}[1]{$^{#1}$}
\begin{thebibliography}{ACBDK15b}

\bibitem[ABG49]{arrow1949bayes}
Kenneth~J Arrow, David Blackwell, and Meyer~A Girshick.
\newblock Bayes and minimax solutions of sequential decision problems.
\newblock {\em Econometrica, Journal of the Econometric Society}, pages
  213--244, 1949.

\bibitem[ACB98]{auer1998line}
Peter Auer and Nicolo Cesa-Bianchi.
\newblock On-line learning with malicious noise and the closure algorithm.
\newblock {\em Annals of mathematics and artificial intelligence},
  23(1):83--99, 1998.

\bibitem[ACBDK15a]{ACDK15}
Noga Alon, Nicolo Cesa-Bianchi, Ofer Dekel, and Tomer Koren.
\newblock Online learning with feedback graphs: Beyond bandits.
\newblock In {\em Annual Conference on Learning Theory}, volume~40. Microtome
  Publishing, 2015.

\bibitem[ACBDK15b]{alon2015online}
Noga Alon, Nicol{\`o} Cesa-Bianchi, Ofer Dekel, and Tomer Koren.
\newblock Online learning with feedback graphs: Beyond bandits.
\newblock {\em arXiv e-prints}, pages arXiv--1502, 2015.

\bibitem[ACBF02]{auer2002finite}
Peter Auer, Nicolo Cesa-Bianchi, and Paul Fischer.
\newblock Finite-time analysis of the multiarmed bandit problem.
\newblock {\em Machine learning}, 47(2):235--256, 2002.

\bibitem[ACBFS02]{auer2002nonstochastic}
Peter Auer, Nicolo Cesa-Bianchi, Yoav Freund, and Robert~E Schapire.
\newblock The nonstochastic multiarmed bandit problem.
\newblock {\em SIAM journal on computing}, 32(1):48--77, 2002.

\bibitem[ACBG{\etalchar{+}}17]{alon2017nonstochastic}
Noga Alon, Nicolo Cesa-Bianchi, Claudio Gentile, Shie Mannor, Yishay Mansour,
  and Ohad Shamir.
\newblock Nonstochastic multi-armed bandits with graph-structured feedback.
\newblock {\em SIAM Journal on Computing}, 46(6):1785--1826, 2017.

\bibitem[ACBGM13]{alon2013bandits}
Noga Alon, Nicol{\`o} Cesa-Bianchi, Claudio Gentile, and Yishay Mansour.
\newblock From bandits to experts: A tale of domination and independence.
\newblock {\em Advances in Neural Information Processing Systems},
  26:1610--1618, 2013.

\bibitem[AMM19]{arora2019bandits}
Raman Arora, Teodor~V Marinov, and Mehryar Mohri.
\newblock Bandits with feedback graphs and switching costs.
\newblock {\em Advances in Neural Information Processing Systems}, 32, 2019.

\bibitem[BES13]{buccapatnam2013multi}
Swapna Buccapatnam, Atilla Eryilmaz, and Ness~B Shroff.
\newblock Multi-armed bandits in the presence of side observations in social
  networks.
\newblock In {\em 52nd IEEE Conference on Decision and Control}, pages
  7309--7314. IEEE, 2013.

\bibitem[BES14]{buccapatnam2014stochastic}
Swapna Buccapatnam, Atilla Eryilmaz, and Ness~B Shroff.
\newblock Stochastic bandits with side observations on networks.
\newblock In {\em The 2014 ACM international conference on Measurement and
  modeling of computer systems}, pages 289--300, 2014.

\bibitem[BKNS12]{BKNS12}
Nikhil Bansal, Nitish Korula, Viswanath Nagarajan, and Aravind Srinivasan.
\newblock Solving packing integer programs via randomized rounding with
  alterations.
\newblock {\em Theory of Computing}, 8(1):533--565, 2012.

\bibitem[BV04]{BV04}
Stephen Boyd and Lieven Vandenberghe.
\newblock {\em Convex optimization}.
\newblock Cambridge university press, 2004.

\bibitem[CBL06]{cesa2006prediction}
Nicolo Cesa-Bianchi and G{\'a}bor Lugosi.
\newblock {\em Prediction, learning, and games}.
\newblock Cambridge university press, 2006.

\bibitem[CHK16]{cohen2016online}
Alon Cohen, Tamir Hazan, and Tomer Koren.
\newblock Online learning with feedback graphs without the graphs.
\newblock In {\em International Conference on Machine Learning}, pages
  811--819. PMLR, 2016.

\bibitem[Cov99]{COV99}
Thomas~M Cover.
\newblock {\em Elements of information theory}.
\newblock John Wiley \& Sons, 1999.

\bibitem[EBSSG12]{eban2012learning}
Elad Eban, Aharon Birnbaum, Shai Shalev-Shwartz, and Amir Globerson.
\newblock Learning the experts for online sequence prediction.
\newblock 2012.

\bibitem[Fan61]{Fan61}
Robert~M Fano.
\newblock Transmission of information: A statistical theory of communications.
\newblock {\em American Journal of Physics}, 29(11):793--794, 1961.

\bibitem[FL18]{feng2018online}
Zhili Feng and Po-Ling Loh.
\newblock Online learning with graph-structured feedback against adaptive
  adversaries.
\newblock In {\em 2018 IEEE International Symposium on Information Theory
  (ISIT)}, pages 931--935. IEEE, 2018.

\bibitem[FS97]{freund1997decision}
Yoav Freund and Robert~E Schapire.
\newblock A decision-theoretic generalization of on-line learning and an
  application to boosting.
\newblock {\em Journal of computer and system sciences}, 55(1):119--139, 1997.

\bibitem[KNVM14]{kocak2014efficient}
Tom{\'a}{\v{s}} Koc{\'a}k, Gergely Neu, Michal Valko, and R{\'e}mi Munos.
\newblock Efficient learning by implicit exploration in bandit problems with
  side observations.
\newblock In {\em Neural Information Processing Systems}, pages 613--621, 2014.

\bibitem[LBS18]{liu2018information}
Fang Liu, Swapna Buccapatnam, and Ness Shroff.
\newblock Information directed sampling for stochastic bandits with graph
  feedback.
\newblock In {\em Proceedings of the AAAI Conference on Artificial
  Intelligence}, volume~32, 2018.

\bibitem[LCWL20]{li2020stochastic}
Shuai Li, Wei Chen, Zheng Wen, and Kwong-Sak Leung.
\newblock Stochastic online learning with probabilistic graph feedback.
\newblock In {\em Proceedings of the AAAI Conference on Artificial
  Intelligence}, volume~34, pages 4675--4682, 2020.

\bibitem[LLZ20]{lee2020closer}
Chung-Wei Lee, Haipeng Luo, and Mengxiao Zhang.
\newblock A closer look at small-loss bounds for bandits with graph feedback.
\newblock In {\em Conference on Learning Theory}, pages 2516--2564. PMLR, 2020.

\bibitem[MS11]{mannor2011bandits}
Shie Mannor and Ohad Shamir.
\newblock From bandits to experts: on the value of side-observations.
\newblock In {\em Proceedings of the 24th International Conference on Neural
  Information Processing Systems}, pages 684--692, 2011.

\bibitem[MT04]{MT04}
Shie Mannor and John~N Tsitsiklis.
\newblock The sample complexity of exploration in the multi-armed bandit
  problem.
\newblock {\em Journal of Machine Learning Research}, 5(Jun):623--648, 2004.

\bibitem[Nem79]{nemirovski1979efficient}
Arkadi Nemirovski.
\newblock Efficient methods for large-scale convex optimization problems.
\newblock {\em Ekonomika i Matematicheskie Metody}, 15(1), 1979.

\bibitem[NY83]{nemirovskij1983problem}
Arkadij~Semenovi{\v{c}} Nemirovskij and David~Borisovich Yudin.
\newblock Problem complexity and method efficiency in optimization.
\newblock 1983.

\bibitem[Ora19]{Ora2019}
Francesco Orabona.
\newblock A modern introduction to online learning.
\newblock {\em arXiv preprint arXiv:1912.13213}, 2019.

\bibitem[RF19]{rangi2019online}
Anshuka Rangi and Massimo Franceschetti.
\newblock Online learning with feedback graphs and switching costs.
\newblock In {\em The 22nd International Conference on Artificial Intelligence
  and Statistics}, pages 2435--2444. PMLR, 2019.

\bibitem[RKJ08]{radlinski2008learning}
Filip Radlinski, Robert Kleinberg, and Thorsten Joachims.
\newblock Learning diverse rankings with multi-armed bandits.
\newblock In {\em Proceedings of the 25th international conference on Machine
  learning}, pages 784--791, 2008.

\bibitem[Sri99]{Sri99}
Aravind Srinivasan.
\newblock Improved approximation guarantees for packing and covering integer
  programs.
\newblock {\em SIAM Journal on Computing}, 29(2):648--670, 1999.

\bibitem[TDD17]{tossou2017thompson}
Aristide Tossou, Christos Dimitrakakis, and Devdatt Dubhashi.
\newblock Thompson sampling for stochastic bandits with graph feedback.
\newblock In {\em Proceedings of the AAAI Conference on Artificial
  Intelligence}, volume~31, 2017.

\bibitem[Tsa88]{tsallis1988possible}
Constantino Tsallis.
\newblock Possible generalization of boltzmann-gibbs statistics.
\newblock {\em Journal of statistical physics}, 52(1):479--487, 1988.

\bibitem[Wal47]{wald1947sequential}
Abraham Wald.
\newblock Sequential analysis.
\newblock 1947.

\bibitem[ZL19]{ZL19}
Julian Zimmert and Tor Lattimore.
\newblock Connections between mirror descent, thompson sampling and the
  information ratio.
\newblock In {\em Advances in Neural Information Processing Systems}, pages
  11973--11982, 2019.

\end{thebibliography}

\begin{appendix}
	\section{Two Definitions of Weak Dominating Set}\label{apx:dc}
	Here we give an example to explain the difference between the definition of the weak domination number $\delta$ in~\cite{ACDK15} and in this article. To ease the presentation, we call them $\delta$ and $\delta'$ respectively. We show that $\delta=\delta'$ if $\delta\ge 2$, but it is possible that $\delta=1$ and $\delta'=2$.
		\begin{center}
		\begin{figure}[h]
			
			\begin{minipage}{0.5\textwidth}
				\begin{tikzpicture}
				\def\r{3}
				\def\c{4}
				\begin{scope}[every node/.style={circle,thick,draw, minimum size=0.5cm}]
				\node (A) at (\c-\r/2,\c+\r/2) {A};
				\node (B) at (\c-\r/2,\c-\r/2) {B};
				\node (C) at (\c,\c) {C};
				\node (D) at (\c+\r,\c) {D};
				\end{scope}
				\begin{scope}[>={Stealth[black]},
				every node/.style={fill=white,circle},
				every edge/.style={draw=red,very thick}]
				\path [->] (A) edge  (C);
				\path [->] (B) edge  (C);
				\path [->,draw=red,very thick] (C) to[bend right]  (D);
				\path [->,draw=red,very thick] (D) to[bend right]  (C);
				\path [->,draw=red,very thick] (A) to[loop below] (A);
				\path [->,draw=red,very thick] (B) to[loop above] (B);
				\end{scope}
				\end{tikzpicture}
			\end{minipage}
			\caption{An example of $\delta=1$, $\delta'=2$}
			\label{fig:graph}
		\end{figure}
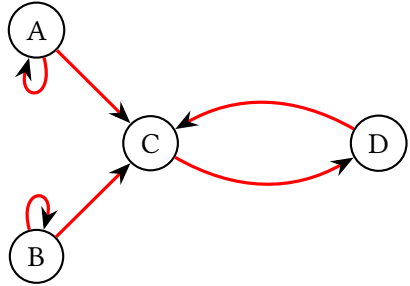
	\end{center}
	
It is clear that $\delta=1$ and $\delta'=2$ in \Cref{fig:graph} since the vertex $C$ is strongly observable and in~\cite{ACDK15}, the vertex $C$ does not need to be dominated by a weak dominating set. Therefore, the minimum weak dominating set is $\set{C}$. However, in the proof of \cite[Theorem 2]{ACDK15} for the weakly observable graphs, they assumed that every vertex without a self-loop is dominated by the weakly dominating set. This is not true following their definition since the vertex $C$, although strongly observable, is not dominated by itself and thus the lower bound on the probability that $C$ is observed fails. 

Hence we ask for the set to dominate the set of vertices without self-loops, namely $U=\set{C,D}$. The proof can then go through, and the only difference is that $\delta'$ becomes to two. It is also clear when $\delta\ge 2$, this situation will not occur as every strongly observable vertex without a self-loop can be dominated by any vertex other than itself.

% This situation only occurs when the weakly dominating set only consists of one strongly observable vertex without self-loop and then this vertex will be not dominated and in this situation, $\delta=1$ and $\delta^*=2$. When $\delta\geq2$, this situation will not occur and the weakly dominating set dominating the weakly observable set also dominates $U$ because the strongly observable vertex in $U$ will be dominated by at least one another vertex except for itself. Hence $\delta$ is the optimal of linear program $\mathscr{P^\prime}$ in \Cref{sec:alon} when $\delta\geq2$.
%	
\end{appendix}

\end{document}